\documentclass{article}

 \usepackage[preprint]{neurips_2026}

\usepackage[utf8]{inputenc} 
\usepackage[T1]{fontenc}    
\usepackage{xcolor}
\definecolor{darkblue}{RGB}{0,51,102}
\definecolor{darkgreen}{RGB}{0,85,34}
\usepackage{wrapfig}
\usepackage[
  colorlinks=true,
  linkcolor=darkblue,
  citecolor=darkgreen,
  urlcolor=darkblue
]{hyperref}
\usepackage{url}
\usepackage{booktabs}
\usepackage{amsfonts}
\usepackage{nicefrac}
\usepackage{microtype}
\usepackage{amsmath} 
\usepackage{amsthm}
\usepackage{amssymb}
\usepackage{graphicx}
\usepackage{diagbox} 
\usepackage{makecell}
\usepackage{bm}
\usepackage[capitalise]{cleveref}
\usepackage{enumitem}
\usepackage{algorithmic}
\usepackage{placeins}
\usepackage{tikz}
\usepackage{enumitem}
\usepackage{algorithm}
\usetikzlibrary{decorations.pathreplacing}
\usetikzlibrary{calc,arrows.meta}

\newcommand{\hlpink}[1]{\tikz[baseline=(X.base)]\node[fill=pink!30, inner sep=1.5pt, rounded corners=1.5pt] (X) {#1};}
\newcommand{\roundedcell}[1]{%
  \tikz[baseline=(c.base)]\node[rounded corners=2pt, fill=pink!40, inner xsep=3pt, inner ysep=1pt] (c) {#1};%
}

\newtheorem{theorem}{Theorem}[section]

\newtheorem{lemma}{Lemma}[section]
\newtheorem{assumption}{Assumption}[section]
\newtheorem{proposition}{Proposition}[section]

\newtheorem{definition}{Definition}[section]

\newcommand{\op}{\mathrm{op}} 
\newcommand{\nm}{\mathrm{nm}}      
\newcommand{\eq}{\mathrm{eq}}      

\title{Optimizer-Induced Mode Connectivity: \\From AdamW to Muon}

\author{%
    Fangzhao Zhang$^{*1}$, 
    Sungyoon Kim$^{*1}$, 
    Erica Zhang$^{1}$, 
    Yiqi Jiang$^{1}$, 
    Mert Pilanci$^{1}$ \\[6pt]
    $^{1}$Stanford University 
}

\begin{document}
\maketitle
\renewcommand{\thefootnote}{}
\footnotetext{$^{*}$Equal contribution. Author ordering determined by coin flip.}
\footnotetext{Correspondence to: \texttt{\{zfzhao, sykim777\}@stanford.edu}.}
\renewcommand{\thefootnote}{\arabic{footnote}}

\begin{abstract} 
Mode connectivity has been widely studied, yet the role of the optimizer remains underexplored. We revisit it through optimizer-induced implicit regularization, asking how connectivity behaves when restricted to solutions constrained by a given optimizer. For two-layer ReLU networks, we show that solutions from a single optimizer - AdamW, Muon, or others in the Lion-$\mathcal{K}$ family - form a connected set at sufficiently large width, a result not implied by prior work. We then characterize how optimizer-induced regions interact: at large width two different regions can be disjoint or overlap depending on regularization, while in our small-width example AdamW and Muon converge to disconnected zero-loss components separated by a provable loss barrier. Empirically, in GPT-2 pretraining, we observe same-optimizer paths preserve each model's spectrum while cross-optimizer paths traverse a smooth transition. Our results reveal optimizer-dependent structure beyond classical mode connectivity literature. Code available at: \href{https://github.com/pilancilab/Optimizer-Induced-Mode-Connectivity}{pilancilab/Optimizer-Induced-Mode-Connectivity}.
\end{abstract}

\section{Introduction}
\label{sec:intro}

Mode connectivity~\citep{garipov2018losssurfacesmodeconnectivity} is one of the most intriguing phenomena in neural network loss landscapes: two independently trained models can often be connected by a simple curve - such as a piecewise linear or quadratic path - along which the loss remains low. This observation is unexpected given the non-convex nature of neural network training, and has motivated substantial theoretical work explaining when such low-loss paths exist~\citep{freeman2017topologygeometryhalfrectifiednetwork,kuditipudi2020explaininglandscapeconnectivitylowcost,nguyen2021note,nguyen2021solutions,kim2024exploring,ferbach2024provinglinearmodeconnectivity,zhan2025analyzing}. Beyond its theoretical interest, mode connectivity has enabled practical applications including model fusion, weight re-basin, and efficient ensembling~\citep{pena2022rebasinimplicitsinkhorndifferentiation,liu2022deep,akash2022wassersteinbarycenterbasedmodelfusion,ainsworth2023gitrebasinmergingmodels,jordan2023repairrenormalizingpermutedactivations,singh2023modelfusionoptimaltransport}.

Despite substantial progress in understanding mode connectivity, the role of the optimizer in shaping
connectivity has received little attention. Most existing works study connectivity between solutions
obtained under a fixed training procedure, or focus on properties of the loss landscape independent of
optimization dynamics. If all optimizers converge to the same set of solutions, the choice of optimizer
would be irrelevant for studying connectivity; however, in practice, different optimizers may converge
to qualitatively different solution sets ~\citep{wang2025muonoutperformsadamtailend, jianlinblog}.

Recent introduction of the Muon optimizer~\citep{jordan2024muon} makes this divergence concrete. \cref{fig:motivation} shows that AdamW and Muon converge to solutions with visibly different singular value spectra: AdamW exhibits outliers, while Muon’s spectrum is more concentrated. When we interpolate between two models trained with the same optimizer, the intermediate models all share same spectral signature. Yet, naive linear interpolation between them incurs a substantial loss barrier - leaving open whether the two endpoints can in fact be connected by a low-loss path of models with similar spectrum.

\begin{figure}[t]
    \centering
    \includegraphics[width=0.95\linewidth]{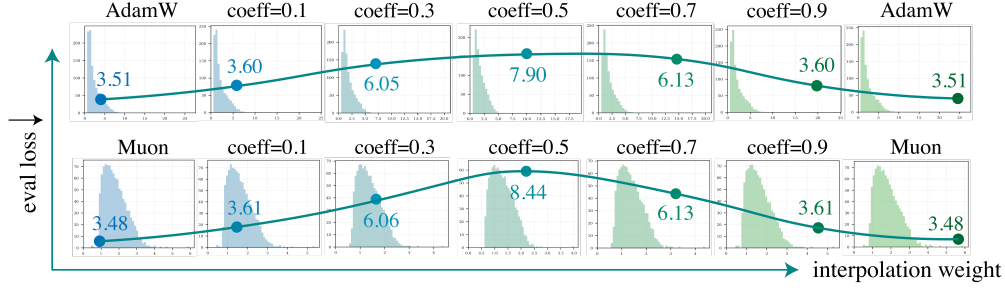}
    \caption{\textbf{Motivating Experiment.} Singular value histograms for layer 1 \texttt{up$\_$proj} weights in GPT-2 training. Models trained with AdamW and Muon converge to solutions with distinct spectrum. Solutions obtained by the same optimizer can be interpolated by paths that preserve the spectrum. Naive interpolation does not connect models with low loss paths. See \cref{sec::inter_exp} for experimental details.}
    \label{fig:motivation}
\end{figure}


Mode connectivity treats the solution set as a whole; it does not address what happens when we restrict to solutions consistent with a given optimizer's implicit bias. We propose a refinement: different optimizers constrain solutions to different regions of the loss landscape - we ask whether these regularized solution sets are connected. This reveals additional structure invisible at the level of classical mode connectivity~\citep{nguyen2021solutions, itowe}. Beyond theoretical interest, low-loss paths between solutions found by different optimizers produce structurally diverse intermediate models that may improve downstream properties such as quantization or generalization. Specifically,

\begin{itemize}[left=0pt]
    \item \textbf{Mode connectivity under optimizer-induced constraints.}
    We introduce the notions of \emph{intra-optimizer} and \emph{inter-optimizer} mode connectivity (\cref{sec:theory}). For AdamW and several widely used optimizers in the Lion-$\mathcal{K}$ framework including Muon, we show that intra-optimizer connectivity holds for two-layer ReLU networks if the width is sufficiently large. Moreover, we characterize conditions under which inter-optimizer connectivity holds.

    \item \textbf{A finite-width disconnectivity construction.}
    We construct a finite-width example in two-layer ReLU network training where the zero-loss solution set decomposes into multiple connected components. We show that there exists a problem where two solutions found from the same optimizer can provably be connected with a zero-loss path, whereas solutions from different optimizers cannot be connected without incurring loss.
    \item \textbf{Empirical evidence in large language model pretraining.}
    We provide empirical evidence from LLM pretraining that is consistent with our perspective. We find that same-optimizer pairs (Muon–Muon, AdamW–AdamW) are connected by low-loss paths that preserve their spectrum, whereas cross-optimizer pairs (Muon–AdamW) remain connected by a low-loss path while exhibiting a clear spectral transition.
    \item \textbf{Cross-optimizer model merging.} We show that interpolating between solutions from different optimizers traverses diverse weight spectra and can improve out-of-distribution generalization.
\end{itemize}


\begin{figure*}[!ht]
\centering
\begin{tikzpicture}[
    scale=0.9,
    every node/.style={font=\small},
    line cap=round,
    line join=round
]

\def\ptsize{2.7pt}
\def\aptsize{3.0pt}

\coordinate (A) at (0,0);
\coordinate (B) at (4.9,0);
\coordinate (C) at (10.0,0);

\fill[green!20, opacity=0.25] ($(A)+(0,0)$) ellipse (1.75 and 1.12);
\draw[line width=1.15pt] ($(A)+(0,0)$) ellipse (1.75 and 1.12);

\fill[green!60!black] ($(A)+(-0.78,0.02)$) circle (\aptsize);
\fill[green!60!black] ($(A)+(0.82,0.18)$) circle (\aptsize);
\draw[green!60!black, thick] ($(A)+(-0.78,0.02)$) -- ($(A)+(0.82,0.18)$);

\node[font=\footnotesize, text=green!30!black] at ($(A)+(0,-0.52)$) {Solution Set};
\node at ($(A)+(0,-1.55)$) {(a) Classical View};

\fill[green!20, opacity=0.25] ($(B)+(0,0)$) ellipse (1.90 and 1.18);
\draw[line width=1.15pt] ($(B)+(0,0)$) ellipse (1.90 and 1.18);

\fill[blue!35, opacity=0.30] ($(B)+(-0.78,0.12)$) ellipse (0.60 and 0.40);
\draw[blue, line width=1.1pt] ($(B)+(-0.78,0.12)$) ellipse (0.60 and 0.40);

\fill[red!35, opacity=0.30] ($(B)+(0.82,0.10)$) ellipse (0.50 and 0.50);
\draw[red, line width=1.1pt] ($(B)+(0.82,0.10)$) ellipse (0.50 and 0.50);

\fill[blue] ($(B)+(-1.00,0.20)$) circle (\ptsize);
\fill[blue] ($(B)+(-0.58,0.02)$) circle (\ptsize);
\draw[blue, thick] ($(B)+(-1.00,0.20)$) -- ($(B)+(-0.58,0.02)$);

\fill[red] ($(B)+(0.68,-0.05)$) circle (\ptsize);
\fill[red] ($(B)+(0.96,0.23)$) circle (\ptsize);
\draw[red, thick] ($(B)+(0.68,-0.05)$) -- ($(B)+(0.96,0.23)$);

\draw[line width=1.0pt, black!65]
  ($(B)+(-0.58,0.02)$) -- ($(B)+(0.68,-0.05)$);

\node[blue, font=\footnotesize] at ($(B)+(-0.78,-0.58)$) {AdamW};
\node[red, font=\footnotesize] at ($(B)+(0.82,-0.58)$) {Muon};

\draw[line width=1.1pt, black!70]
  ($(B)+(-0.35,0.55)$)
  .. controls ($(B)+(-0.15,0.68)$) and ($(B)+(0.20,0.68)$)
  .. ($(B)+(0.42,0.55)$);

\node[font=\footnotesize, text=black] at ($(B)+(0.03,0.88)$) {may intersect};

\node at ($(B)+(0,-1.55)$) {(b) Optimizer-aware View};

\fill[green!20, opacity=0.25] ($(C)+(-1.35,0.08)$) ellipse (1.10 and 0.68);
\draw[line width=1.15pt] ($(C)+(-1.35,0.08)$) ellipse (1.10 and 0.68);

\fill[green!20, opacity=0.25] ($(C)+(1.45,0.08)$) ellipse (1.10 and 0.68);
\draw[line width=1.15pt] ($(C)+(1.45,0.08)$) ellipse (1.10 and 0.68);

\fill[blue!35, opacity=0.38] ($(C)+(-1.35,0.12)$) ellipse (0.70 and 0.38);
\draw[blue, line width=1.1pt] ($(C)+(-1.35,0.12)$) ellipse (0.70 and 0.38);

\fill[red!35, opacity=0.38] ($(C)+(1.45,0.12)$) ellipse (0.70 and 0.38);
\draw[red, line width=1.1pt] ($(C)+(1.45,0.12)$) ellipse (0.70 and 0.38);

\fill[blue] ($(C)+(-1.68,0.22)$) circle (\ptsize);
\fill[blue] ($(C)+(-1.02,0.02)$) circle (\ptsize);
\draw[blue, thick] ($(C)+(-1.68,0.22)$) -- ($(C)+(-1.02,0.02)$);

\fill[red] ($(C)+(1.12,0.02)$) circle (\ptsize);
\fill[red] ($(C)+(1.78,0.22)$) circle (\ptsize);
\draw[red, thick] ($(C)+(1.12,0.02)$) -- ($(C)+(1.78,0.22)$);

\draw[dashed, line width = 1pt] ($(C)+(-1.02,0.02)$) -- ($(C)+(1.12,0.02)$);

\node[blue, font=\footnotesize] at ($(C)+(-1.35,-0.90)$) {AdamW};
\node[red, font=\footnotesize] at ($(C)+(1.45,-0.90)$) {Muon};

\node at ($(C)+(0.05,-1.55)$) {(c) Barrier between optimizers};

\end{tikzpicture}

\caption{\textbf{Mode Connectivity Under Optimizer-Induced Constraints.} Solid line denotes low-loss path, dashed line denotes path with a barrier.
(a) Classical mode connectivity treats low-loss solutions as lying in a connected set.
(b) Regularized solution sets induced by each optimizer are connected for sufficiently wide networks. The regions may or may not intersect, depending on the problem and regularization.
(c) In the finite-width construction, the solution set can decompose into disconnected components selected by different optimizer-induced regularizations.}
\label{fig:refined_view}
\end{figure*}
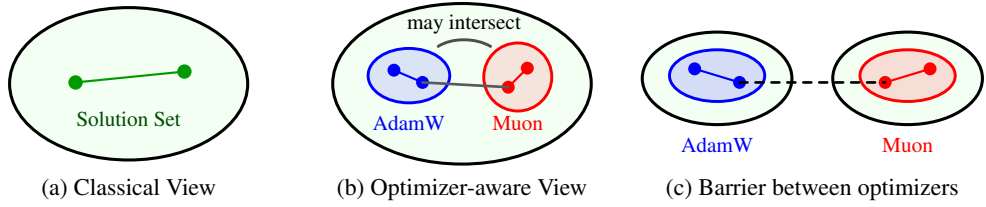

\section{Background}
\label{sec:background} 

\subsection{Notation}
\label{sec:notation}
For a matrix $W \in \mathbb{R}^{d \times m}$ with singular values $\sigma_1(W) \geq \sigma_2(W) \geq \cdots$, we use the entrywise max norm $\|W\|_{\max} := \max_{i,j} |W_{ij}|$, the entrywise $\ell_1$ norm $\|W\|_1 := \sum_{i,j} |W_{ij}|$, the Frobenius norm $\|W\|_F$, the spectral norm $\|W\|_{\mathrm{op}} := \sigma_1(W)$, and the nuclear norm $\|W\|_* := \sum_i \sigma_i(W)$. We avoid $\|\cdot\|_\infty$ on matrices, which can mean either the entrywise max or the induced operator norm. For vectors, $\|v\|_p$ denotes the standard $\ell_p$ norm. 

We write a two-layer ReLU network of width $m$ as $f(x) = (xW)_+ \alpha$, where $(\cdot)_+ := \max(\cdot, 0)$ is the ReLU function. Throughout, $W_i \in \mathbb{R}^d$ denotes the $i$-th column of $W$ (the first-layer weights of neuron $i$), and $\alpha_i \in \mathbb{R}$ is the corresponding second-layer weight for $i = 1, 2, \cdots m$. 

For $h \in \mathbb{R}^d$, the sign vector $\mathbf{1}(Xh \geq 0) \in \{0,1\}^n$ records which inputs activate a neuron with first-layer weight $h$. We call these \emph{hyperplane arrangement patterns} and write $D_1, \ldots, D_P$ for the $P$ distinct ones realized as $h$ varies over $\mathbb{R}^d$. The integer $P$ depends only on $X$ and is $\mathcal{O}(n^r)$, where $r = \mathrm{rank}(X)$. Finally, $[m] := \{1, \ldots, m\}$ and $\arg\min$ denotes the set of all minimizers.

\subsection{Adaptive Optimizers and Their Implicit Bias}
\label{sec:imbias}
Preconditioned optimization methods have become a dominant paradigm in modern deep learning. AdamW~\citep{loshchilov2019decoupledweightdecayregularization} decouples weight decay from the adaptive update and has become the standard optimizer for large-scale neural network training. Beyond the diagonal preconditioning that AdamW performs, matrix-preconditioned methods such as Shampoo~\citep{gupta2018shampoopreconditionedstochastictensor} and SOAP~\citep{vyas2025soapimprovingstabilizingshampoo} use richer second-moment structure. Muon~\citep{jordan2024muon} has brought renewed attention to this direction by using matrix orthogonalization via Newton-Schulz iterations, with subsequent work demonstrating its scalability for language model training~\citep{liu2025muonscalablellmtraining,ai2025practicalefficiencymuonpretraining,kimiteam2025kimik2openagentic}.


These optimizers differ in the geometry of their updates, leading to qualitatively different solutions at convergence. Prior work showed that adaptive optimizers can converge to solutions with different generalization behavior than SGD~\citep{wilson2017marginal}. Theoretical studies have characterized the implicit bias of adaptive methods in homogeneous models~\citep{wang2021implicit}, and extended implicit regularization results to Adam~\citep{cattaneo2023implicit} and Muon~\citep{fan2025implicitbiasspectraldescent}. Most relevant to our work is the perspective of identifying limit points of adaptive algorithms as KKT points of constrained optimization. \citet{xie2024implicitbiasadamwellinfty} showed that if AdamW with updates
\[
\theta_{t+1}
=
\theta_t-\eta_t\!\left(
\frac{m_{t+1}}{\sqrt{v_{t+1}}+\epsilon}
+\lambda \theta_t
\right),
\]
where $m_{t+1}=\beta_1 m_t+(1-\beta_1)\nabla f(\theta_t)$ and $v_{t+1}=\beta_2 v_t+(1-\beta_2)\nabla f(\theta_t)^{2}$ converges, it converges to $\ell_\infty$-norm constrained solutions. More generally, the Lion-$\mathcal{K}$ framework~\citep{chen2025lionsecretlysolvesconstrained} considers updates
\[
m_{t+1} = \mu m_t + \nabla f(\theta_t), \qquad
\theta_{t+1}
=
\theta_t
-
\eta_t\big(\nabla K(m_{t+1}) + \lambda \theta_t\big),
\]
for some convex function $K$ with subgradient $\nabla K$. If $K$ is a norm, the limit points satisfy the dual-norm constraint $K_d(\theta) \le 1/\lambda$. In particular, Muon corresponds to $K = \|\cdot\|_*$~\citep{chen2025muonoptimizesspectralnorm}, the nuclear norm, yielding the operator norm constraint $\|\theta\|_{\mathrm{op}} \le 1/\lambda$.
\subsection{Mode Connectivity}
Mode connectivity was first observed by~\citet{garipov2018losssurfacesmodeconnectivity} and~\citet{draxler2019essentiallybarriersneuralnetwork}, and has since been studied extensively. The simplest form is linear mode connectivity (LMC), first studied by~\citet{frankle2020linearmodeconnectivitylottery}, where the path to interpolate two models $\theta_A$ and $\theta_B$ is the line segment $\phi_\theta(t)=(1-t)\theta_A + t\theta_B$. While LMC typically fails in raw parameter space, it often emerges after accounting for symmetries of the model. Neural networks admit parameter transformations - such as neuron permutations or orthogonal reparametrizations - that preserve the represented function~\citep{zhao2023symmetryteleportationacceleratedoptimization,zhao2024improvingconvergencegeneralizationusing,zhao2025understandingmodeconnectivityparameter,theus2025generalizedlinearmodeconnectivity,tran2025linearmodeconnectivitymixtureofexperts}. The incorporation of permutation symmetry into LMC study~\citep{entezari2022rolepermutationinvariancelinear} has led to practical re-basin methods~\citep{pena2022rebasinimplicitsinkhorndifferentiation,ainsworth2023gitrebasinmergingmodels,singh2023modelfusionoptimaltransport}. More recent work has extended mode connectivity analysis to Transformer architectures~\citep{imfeld2024transformerfusionoptimaltransport,verma2024mergingtexttransformermodels,theus2025generalizedlinearmodeconnectivity}.

\section{Mode Connectivity Under Optimizer-Induced Constraints}
\label{sec:theory}
\subsection{Conceptual Framework}
\label{sec:framework}

Consider the optimization problem $\min_{\theta \in \mathbb{R}^{p}} f(\theta)$. As discussed in \cref{sec:background}, implicit-bias results identify the limit points of AdamW and Lion-$\mathcal{K}$ optimizers as KKT points of norm-constrained problems. These results motivate studying the set of optimal solutions satisfying the corresponding implicit constraints, since any limit point found by AdamW or Lion-$\mathcal{K}$ must satisfy them.
\begin{definition}[Regularized and optimizer-regularized solution sets]
\label{def:regsolset}
For a norm $R$ and $\lambda > 0$, the regularized solution set associated with $R$ is
\[
\mathcal{O}_{R}(\lambda)
:=
\operatorname*{arg\,min}_{\theta \in \mathbb{R}^p} f(\theta)
\;\cap\;
\{
\theta \in \mathbb{R}^p
\ | \ 
R(\theta) \le 1/\lambda
\}.
\]
For AdamW or any Lion-$\mathcal{K}$ optimizer $\mathcal{A}$ run with weight decay $\lambda$, the optimizer-regularized solution set is $\mathcal{O}^\mathcal{A}(\lambda) := \mathcal{O}_{\kappa_\mathcal{A}}(\lambda)$, where $\kappa_\mathcal{A}$ is the norm induced by $\mathcal{A}$'s implicit bias.
\end{definition}
Concretely, $\kappa_{\mathrm{AdamW}} = \|\cdot\|_{\max}$ and $\kappa_{\mathrm{Muon}} = \|\cdot\|_{\mathrm{op}}$, while a general Lion-$\mathcal{K}$ optimizer $L_K$ associated with norm $K$ gives $\kappa_{L_K} = K_d$. We use a norm subscript ($\mathcal{O}_R$) when the norm is the primary object and an optimizer superscript ($\mathcal{O}^\mathcal{A}$) when the optimizer is the primary object. By construction, $\mathcal{O}^\mathcal{A}(\lambda)$ contains the limit points of $\mathcal{A}$ as a subset; we study this regularized set as a structural object in its own right rather than attempting to identify the limit-point set exactly.

With this object at hand, the question we want to ask becomes precise: \emph{is $\mathcal{O}^\mathcal{A}(\lambda)$ connected?} Two natural notions arise: connectivity within one optimizer's regularized set, and connectivity across different optimizers' sets.
\begin{definition}[Intra- and inter-optimizer connectivity]
The loss landscape is \textbf{intra-optimizer connected} for an optimizer $\mathcal{A}$ and regularization $\lambda > 0$ if $\mathcal{O}^\mathcal{A}(\lambda)$ is connected. It is \textbf{inter-optimizer connected} for optimizers $\mathcal{A}_1, \mathcal{A}_2$ and regularizations $\lambda_1, \lambda_2 > 0$ if $\mathcal{O}^{\mathcal{A}_1}(\lambda_1) \cup \mathcal{O}^{\mathcal{A}_2}(\lambda_2)$ is connected.  
\end{definition}
Two consequences of these definitions are worth highlighting. Intra-optimizer connectivity implies that any two limit points of $\mathcal{A}$ can be joined by a continuous path along which the implicit-bias constraint is preserved. The reverse direction does not hold in general - connectivity of $\mathcal{O}^\mathcal{A}(\lambda)$ is a necessary but not sufficient condition for limit-point connectivity - but disconnectivity is inherited: if $\mathcal{O}^{\mathcal{A}_1}(\lambda_1)$ and $\mathcal{O}^{\mathcal{A}_2}(\lambda_2)$ lie in different components of the loss landscape, then so do the corresponding limit points. For certain subgradient choices in Lion-$\mathcal{K}$, the optimizer-regularized solution set coincides with the limit-point solution set; see Proposition~\ref{prop:exact}.

Although intra-optimizer connectivity is automatic for convex problems, the regularized solution set for two-layer ReLU networks is nonconvex. Existing connectivity results for the unrestricted solution set \citep{freeman2017topologygeometryhalfrectifiednetwork,nguyen2021note} do not directly imply the regularized version.

In the following sections, we specialize this framework to two-layer ReLU networks in the interpolation regime. We focus on four optimizers: AdamW and three instances of Lion-$\mathcal{K}$ optimizers with $K = \|\cdot\|_1$, $\|\cdot\|_F$, and $\|\cdot\|_*$, corresponding to Signum \citep{bernstein2018signsgd}, normalized momentum GD  \citep{cutkosky2020momentum} with decoupled weight decay, and Muon, respectively. Their induced norms are $\kappa_{\mathrm{AdamW}} = \|\cdot\|_{\max}$, $\kappa_{\mathrm{Signum}} = \|\cdot\|_{\max}$, $\kappa_{\mathrm{NormMom}} = \|\cdot\|_F$, and $\kappa_{\mathrm{Muon}} = \|\cdot\|_{\mathrm{op}}$. We show that intra-optimizer connectivity holds for all four at sufficiently large width $m$ (\cref{sec:intra-mode connectivity}). On the other hand, inter-optimizer connectivity depends on the data, the regularization parameters, and the network width (\cref{sec:inter-mode connectivity}). 

\subsection{Intra-Optimizer Connectivity in Two-Layer ReLU Network}
\label{sec:intra-mode connectivity}

We first specialize the regularized solution set to two-layer ReLU networks in the
interpolation regime. Consider
$
\min_{W,\alpha} \mathcal L((XW)_+\alpha,y),
$
where \(X\in\mathbb R^{n\times d}\), \(y\in\mathbb R^n\),
\(W\in\mathbb R^{d\times m}\), \(\alpha\in\mathbb R^m\) and $\mathcal{L}(x,y) = 0 \Leftrightarrow x = y$.
We assume throughout this section that the data $(X,y)$ is realizable by a two-layer ReLU network.

\begin{assumption}[Two-layer realizability]
\label{assump:realizability}
There exists \((W_0,\alpha_0)\) such that $(XW_0)_+\alpha_0 = y.$ Also, when $m_0$ denotes the number of elements in $\alpha_0$, we assume $m \geq m_0$.
\end{assumption}
Under \cref{assump:realizability}, the optimal set of the training problem becomes the set of interpolators.  

The subtlety we need to consider is that the optimizer is applied to the joint variable \((W,\alpha)\), whereas the implicit-bias results in \cite{xie2024implicitbiasadamwellinfty}, \cite{chen2025muonoptimizesspectralnorm} are stated for a single parameter vector or matrix. For a Lion-$\mathcal{K}$ optimizer with norm $K$, we must consider the regularized solution set 
\begin{equation}
\label{eq:twolayerReg}
\mathcal{O}_{K_d}(\lambda)
=
\{
(W,\alpha) \ | \
(XW)_+\alpha = y,\;
\max\{K_d(W),K_d(\alpha)\}\le 1/\lambda\}. 
\end{equation}

\cref{appxprop2:decoupleweights} rigorously discusses why \cref{eq:twolayerReg} becomes the right regularized solution set when we train two-layer networks. The idea is when we apply the Lion-$\mathcal K$ optimization iteration separately for $W$ and $\alpha$, it is effectively equivalent to applying $K(W)+K(\alpha)$ to the block variable $(W,\alpha)$. Hence the dual norm becomes $\max \{K_d(W), K_d(\alpha)\}$ and we have \cref{eq:twolayerReg} as the regularized solution set. Similarly we have that for AdamW, the regularized solution set of interest is \cref{eq:twolayerReg} with $K = \|\cdot\|_1$.

We now prove that given $(X, y, \lambda)$, for sufficiently wide network with width $m \geq m_0$, intra-mode connectivity holds for Normalized GD with momentum and Muon (\cref{t1:intramode_rot_31}), as well as Signum and AdamW (\cref{t2:intra_signum_32}). Our proof strategy consists of two steps: first, we show that there exists a set $S$ and a threshold $m^*$ that is only dependent on $(X, y, \lambda)$ such that when $m \ge m^*$, for any $A \in \mathcal{O}_R(\lambda)$ for $R: \|\cdot\|_F, \|\cdot\|_{\max}, \|\cdot\|_{\op}$, there exists a path from $A$ to $S \subseteq \mathcal{O}_R(\lambda)$ that satisfies:
\[ (W, \alpha) \in S \implies \sum_{i=1}^m 1(W_i \neq 0) \le m^*/2, \quad W_i \alpha_i = 0 \quad \textrm{implies} \quad W_i = 0, \alpha_i = 0. \]
This step reduces the number of nonzero first-layer weights to a finite threshold. The second step is connecting the two solutions in $S$ (\cref{Lemma1:smaller}). The idea behind \cref{Lemma1:smaller} is that $S$ is sparse enough so that simple interpolation connects the two points in $S$ after appropriate permutation. 

\begin{lemma}
\label{Lemma1:smaller}
Let $m \ge m^*$ and $R: \|\cdot\|_F, \|\cdot\|_{\max}, \|\cdot\|_{\op}$. Let $S \subseteq \mathcal{O}_R(\lambda)$ satisfies $(W, \alpha) \in S \implies \sum_{i=1}^m 1(w_i \neq 0) \le m^*/2$. For all $(W, \alpha) \in S$ and any two points $A \ne B \in S$, there exists a continuous path in $\mathcal{O}_R(\lambda)$ that connects $A$ and $B$. 
\end{lemma}

Now we show how to find the specific $S$ and $m^*$ for each regularization. For $R: \|\cdot\|_F$ and $\|\cdot\|_{\op}$, we choose $S$ to be $\mathcal{O}_R^{\nm}(\lambda)$, the set of non-mergeable solutions. For convenience we define whether a model has inactive neurons removed.
\begin{definition}[Inactive neurons removed]
$(W,\alpha)$ is said to have \emph{inactive neurons removed} if
\(W_i \alpha_i = 0\) implies \(W_i = 0\) and \(\alpha_i = 0\).
\end{definition}
\begin{definition}[Non-mergeable solutions]
The set of non-mergeable solutions $\mathcal{O}_R^{\nm}(\lambda)$ is defined as the set of $(W,\alpha) \in \mathcal{O}_R(\lambda)$ where $(W,\alpha)$ has inactive neurons removed and for any $i \neq j$ such that $w_i \alpha_i \neq 0$ and $w_j \alpha_j \neq 0$,
\(1(XW_i \ge 0) \neq 1(XW_j \ge 0)\) or \((\alpha_i \ge 0) \neq 1(\alpha_j \ge 0).\)
\end{definition}

Hence if we denote $P$ as number of all possible arrangement patterns $\mathbf{1}(Xh \ge 0)$, for each $(W, \alpha) \in \mathcal{O}_R^{nm}(\lambda)$ we have $\sum_{i=1}^m \mathbf{1}(W_i \ne 0) \le 2P$. If we choose $m^*=4P$, we can see that it satisfies the premises in \cref{Lemma1:smaller}. To complete our proof strategy we show that for any $A \in \mathcal{O}_R(\lambda)$, we have a continuous path from $A$ to a point in $\mathcal{O}_R^{\nm}(\lambda)$. The idea is using a nonlinear merging strategy to merge two neurons with same activation and second layer sign, while not violating the norm constraint.

\begin{proposition}
\label{prop3:reduction}
(Reducing to non-mergeable solution) Let $R: \|\cdot\|_F$ or $\|\cdot\|_{\op}$. Then for all $A \in \mathcal{O}_R(\lambda)$, there exists a continuous path from $A$ to $\mathcal{O}_R^{\nm}(\lambda)$ in $\mathcal{O}_R(\lambda)$.
\end{proposition}

\cref{prop3:reduction} and \cref{Lemma1:smaller} lead to the intra-mode connectivity of Normalized GD with Momentum and Muon.

\begin{theorem}
\label{t1:intramode_rot_31}
Intra-optimizer connectivity holds for Normalized GD with Momentum and Muon when $m \ge 4P$. Here $P$ is the number of possible hyperplane arrangement patterns $\mathbf{1}(Xh \ge 0)$. 
\end{theorem}

The case of Signum and AdamW needs a more involved analysis, because it is not immediately clear how one would reduce the solutions as we did for Normalized GD with momentum and Muon. We will define a different set $S$ where all solutions in $\mathcal{O}_{\|\cdot\|_{\max}}(\lambda)$ can move to, as well as the critical width $m^*$. We first define the set of equalized solutions.

\begin{definition}[Equalized solutions]
We define the equalized solution set $\mathcal{O}^{\eq}_{\|\cdot\|_{\max}}(\lambda)$ as the set of 
$(W,\alpha) \in \mathcal{O}_{\|\cdot\|_{\max}}$ where $(W,\alpha)$ has inactive neurons removed, $\alpha_i \in \{0, 1/\lambda, -1/\lambda\}$ and for any $i \neq j$, if
    \(
    \mathbf{1}(XW_i \ge 0) = \mathbf{1}(XW_j \ge 0)
    \)
    \text{and}
    \(
    \alpha_i = \alpha_j, 
    \)
    then \( W_i=W_j.
    \)
\end{definition}

The equalized solution set is the set of solutions where the second layer weights are set to either $0$ or $\pm 1/\lambda$, and if first-layer weights have the same hyperplane arrangement pattern they are identical. Also, we let the $(t, s)$-equalized solution set $\mathcal{O}_{\|\cdot\|_{\max}, (t, s)}^{\eq}$ for a vector $(t,s) \in \mathbb{R}^{2P}$ as:
\[
\begin{aligned}
\mathcal{O}_{\|\cdot\|_{\max}, (t, s)}^{eq} = \big\{ (W, \alpha) \in \mathcal{O}_{\|\cdot\|_{\max}}^{\eq} \mid \textstyle \sum_i &\mathbf{1}(XW_i \ge 0 = D_p) \mathbf{1}(\alpha_i > 0) = t_p, \\
\textstyle & \textstyle \sum_i \mathbf{1}(XW_i \ge 0 = D_p) \mathbf{1}(\alpha_i < 0) = s_p \quad \forall p \in [P] \big\}.
\end{aligned}
\]
$(t, s)$ is a length $2P$ vector of integers that encode the "support" of the sign of $(W, \alpha)$. 

To define the critical width $m^*$, we use a convex reformulation of the network parameterization based on the encoding $(t,s)\in\mathbb{Z}_{\ge0}^{2P}$. Similar ideas have been used to characterize the optimal set of neural networks \citep{mishkin2023optimal,kim2024exploring}.

\begin{definition}
For $(t,s) \in \mathbb{Z}_{\ge 0}^{2P}$, we define the convexified $(t,s)$-solution set
\[
P_{(t,s)} =
\left\{
(u_i,v_i)_{i=1}^P \;\middle|\;
\begin{array}{l}
\sum_{i=1}^P D_i X (u_i - v_i) = y, \|u_i\|_\infty \le \frac{t_i}{\lambda^2}, \quad
\|v_i\|_\infty \le \frac{s_i}{\lambda^2}, \\[3pt]
\mathbf{1}(Xu_i \ge 0) = D_i \ \  \text{if } \  u_i \ne 0, 
\mathbf{1}(Xv_i \ge 0) = D_i \ \ \text{if } \ v_i \ne 0
\end{array}
\right\}.
\]
\end{definition}

The critical threshold $m^*$ will follow from  $P_{(t, s)}$.
Let $Z = \{ (t, s) \mid P_{(t, s)} \ne \emptyset \}$ and let $Z_A$ be the set of minimal elements of $Z$ under the usual partial order. An application of Dickson's lemma shows that $Z_A$ is finite (\cref{Prop5:Dickson}), and each element $p$ of $Z$ has a minimal element that is smaller than or equal to $p$. Hence $Z_A$ can be thought of as a finite set of "smaller" supports one can arrive at from $\mathcal{O}_{\|\cdot\|_{\max}}^{\eq}(\lambda)$. Given this $Z_A$ we can choose $S$, $m^{*}$ to accomplish the first step of our proof.
\begin{proposition}
\label{Prop6:l1reduction}
Let $S = \bigcup_{(t,s) \in Z_A} \mathcal{O}_{\|\cdot\|_{\max}, (t,s)}^{\eq}(\lambda)$ and $m^* = 2 \max_{(t,s) \in Z_A} \sum_{i=1}^P t_i + s_i$. Then the following hold:
\begin{enumerate}[left=0pt]
    \item[i)] Let $m \geq m^*$. If $(W, \alpha) \in S$, $\sum_{i=1}^m 1(W_i \neq 0) \leq m^*/2$ and $(W,\alpha)$ has inactive neurons removed.
    \item[ii)] For all $A \in \mathcal{O}_{\|\cdot\|_{\max}}(\lambda)$, there exists continuous path from $A$ to $S$ in $\mathcal{O}_{\|\cdot\|_{\max}}(\lambda)$.
\end{enumerate}
\end{proposition}
\cref{Prop6:l1reduction} and \cref{Lemma1:smaller} together show intra-mode connectivity of Signum and AdamW for sufficiently wide networks.
\begin{theorem}   
\label{t2:intra_signum_32}
(Intra-optimizer connectivity for Signum and AdamW) Intra-mode connectivity holds for Signum and AdamW when $m\geq m^{*}$, where $m^*$ is only determined with $(X, y, \lambda)$.
\end{theorem}

\subsection{Inter-Optimizer Connectivity: the Role of Width}
\label{sec:inter-mode connectivity}

We now study how the regularized sets $\mathcal{O}_{R_1}(m,\lambda_1)$ and $\mathcal{O}_{R_2}(m,\lambda_2)$ are arranged in the loss landscape, where we write $m$ explicitly to track the dependence on width. Here $R_1 \neq R_2$ and they are from $\{\|\cdot\|_{\max}, \|\cdot\|_F, \|\cdot\|_{\op}\}$. The geometric picture depends sharply on $m$, and we treat the large-width and small-width regimes separately.
\paragraph{Large width regime.}
\cref{appxprop7:regime_FOP,prop9:nontrivial_l1} establish that the regularized solution set $\mathcal{O}_R(m,\lambda)$ is nonempty and connected when $m \geq m^*_R$ and $0 < \lambda \leq \lambda^*_R(m)$. We study how $\mathcal{O}_{R_1}(m,\lambda_1)$ and $\mathcal{O}_{R_2}(m,\lambda_2)$ are arranged when $m \geq \max\{m^*_{R_1}, m^*_{R_2}\}$ and $\lambda_i \leq \lambda^*_{R_i}(m)$ - we refer to this as the \emph{large-width} regime. In this regime, inter-optimizer connectivity reduces to the relative position of $\mathcal{O}_{R_1}(m,\lambda_1)$ and $\mathcal{O}_{R_2}(m,\lambda_2)$. The union $\mathcal{O}_{R_1}(m,\lambda_1) \cup \mathcal{O}_{R_2}(m,\lambda_2)$ is connected if and only if the two sets touch - formally, whether the zero-loss set passes through the overlap of the two norm balls $B(R_1) \cap B(R_2)$, where $B(R) = \{(W,\alpha) \mid \max\{R(W), R(\alpha)\} \leq 1/\lambda\}$. \cref{thm:inter-connectivity} makes this precise: it fixes $\lambda_1$ and varies $\lambda_2$, showing that the number of connected components of the union either stays at one throughout (when the two sets always overlap) or undergoes a single transition from one to two as $\lambda_2$ grows past a critical threshold.

\begin{theorem}
\label{thm:inter-connectivity}
Let $R_1, R_2$ be two regularizers in $\{\|\cdot\|_{\max}, \|\cdot\|_F, \|\cdot\|_{\op}\}$, $m \geq \max\{m_{R_1}^*, m_{R_2}^*\}$, and $0 < \lambda_1 \leq \lambda_{R_1}^*(m)$ fixed. As $\lambda_2$ varies from $0$ to $\lambda_{R_2}^{*}(m)$:
\begin{enumerate}[left=0pt]
    \item[i)] If $\mathcal{O}_{R_1}(\lambda_1) \cap \mathcal{O}_{R_2}(\lambda_{R_2}^{*}(m)) \neq \emptyset$, the union $\mathcal{O}_{R_1}(\lambda_1) \cup \mathcal{O}_{R_2}(\lambda_2)$ has a single connected component for all $\lambda_2$.
    \item[ii)] If $\mathcal{O}_{R_1}(\lambda_1) \cap \mathcal{O}_{R_2}(\lambda_{R_2}^{*}(m)) = \emptyset$, there exists $\lambda_2^*(\lambda_1)$ such that the union has one component when $\lambda_2 \leq \lambda_2^*(\lambda_1)$ and two components when $\lambda_2 > \lambda_2^*(\lambda_1)$.
\end{enumerate}
\end{theorem}

\paragraph{Finite-width construction.}
At a smaller width, a sharper picture is possible: the zero-loss solution set $\mathcal{O} = \{(W,\alpha) \mid (XW)_+\alpha = y\}$ can decompose into multiple connected components, with $\mathcal{O}_{R_1}(m,\lambda)$ and $\mathcal{O}_{R_2}(m,\lambda)$ landing in different components. In this case there is no zero-loss path between them at all - any path joining a solution in $\mathcal{O}_{R_1}(m,\lambda)$ to a solution of $\mathcal{O}_{R_2}(m,\lambda)$ must incur loss, permitting a stronger notion of separation. A general statement of the idea is conveyed in \cref{prop2:connectivity}.

\begin{proposition}
\label{prop2:connectivity}
Let $\mathcal{O}=\{(W,\alpha)\mid (XW)_+\alpha=y\}$ have connected components $\mathcal{C}_1,\ldots,\mathcal{C}_N$, and suppose $\mathcal{O}_{R_1}(\lambda_1) \subseteq \mathcal{C}_i$ and $\mathcal{O}_{R_2}(\lambda_2) \subseteq \mathcal{C}_j$ for some $i \neq j$. Then any two points in the same $\mathcal{O}_{R_\ell}(\lambda_\ell)$ are connected by a continuous path in $\mathcal{O}$, but no continuous path in $\mathcal{O}$ joins a point of $\mathcal{O}_{R_1}(\lambda_1)$ to a point of $\mathcal{O}_{R_2}(\lambda_2)$.
\end{proposition}

We instantiate this for AdamW and Muon. Take $X = [A; -A] \in \mathbb{R}^{2d \times d}$ with $A \in \mathbb{R}^{d \times d}$ invertible, $y \in \mathbb{R}^{2d}$ with nonzero entries, and width $m = 2$. Each vector $w \in \mathbb{R}^d$ yields an activation pattern $(Xw)_+$ with at most $d$ nonzero entries, selecting one side of each pair $(A_{i\cdot}, -A_{i\cdot})$ where $A_{i\cdot}$ denotes the $i$-th row of $A$. When the two neurons must activate on disjoint subsets of the data - which $y$ enforces - different activation patterns correspond to distinct connected components of $\mathcal{O}$, indexed by $\sigma \in \{\pm 1\}^d$. \cref{appxp2:connectedcomponentchar} characterizes these components, and \cref{prop4:connectPandTh} computes the operator-norm and $\ell_\infty$ minimizers among them. By choosing $A$ such that the operator-norm-minimizing component differs from the $\ell_\infty$-minimizing component(\cref{prop5:computation}), we obtain regularization parameters under which AdamW and Muon land in different components. For a formal statement of \cref{t1:Finitewidth_construction} see \cref{appxt1:Finitewidth_construction}.

\begin{theorem}[Informal]
\label{t1:Finitewidth_construction}
There exist a dataset $(X,y)$ and regularization parameters $\lambda_{\mathrm{AdamW}}, \lambda_{\mathrm{Muon}} > 0$ such that the AdamW- and Muon-regularized solution sets lie in different connected components of the zero-loss set $\mathcal{O}$. In particular, same-optimizer solutions are connected by zero-loss paths (up to permutation symmetry), while cross-optimizer solutions cannot be joined by any zero-loss path. Moreover, for $\mathcal{L}(x,y) = \tfrac{1}{2}\|x-y\|_2^2$, any continuous path connecting an AdamW solution to a Muon solution must incur loss at least $\tfrac{1}{2}$ at some point along the path.
\end{theorem}

Together, these results give a unified picture of inter-optimizer connectivity. At large width, $\mathcal{O}_{R_1}(\lambda_1)$ and $\mathcal{O}_{R_2}(\lambda_2)$ live inside a single connected zero-loss set \citep{nguyen2021note}; whether they form one or two components depends on whether the regularized sets touch. At small width, the zero-loss set itself can split, and optimizers can be separated by a genuine loss barrier. Although disconnection of the zero-loss set at finite width is known in classical mode connectivity \citep{freeman2017topologygeometryhalfrectifiednetwork,venturi2019spurious,kuditipudi2020explaininglandscapeconnectivitylowcost, nguyen2021note}, our construction shows that the disconnection can be optimizer-dependent: different optimizer implicit biases provably select different components, separating them by a loss barrier.

\section{Experiments}
\label{sec:experiments}
\subsection{Experimental Setup and Path Construction}
\label{sec:setup}
Following \citep{theus2025generalizedlinearmodeconnectivity}, we use decoder-only GPT-2 models from the HuggingFace \texttt{Trainer} class at two scales: a \emph{small} model ($\sim$30M parameters) and a \emph{large} model ($\sim$50M parameters). Small models are used in this section and \cref{sec:ood}'s experiments while large models are used in \cref{sec::inter_exp} for better spectral probe. We train on five datasets of varying sizes: enwik8, WikiText-103, Stories, BookCorpus, and One Billion Word. All models are trained for 15{,}000 steps, consuming roughly 983M tokens---slightly exceeding the Chinchilla $20\times$ optimal token count for our model sizes \citep{hoffmann2022trainingcomputeoptimallargelanguage}. Learning rate and weight decay are tuned independently per optimizer and per dataset; full details including model configurations, dataset statistics and hyperparameter sweeps are reported in \cref{observation_appendix}. We report/plot mean and std over $5$ random seeds.

Our path construction algorithm is a modification of \cite{theus2025generalizedlinearmodeconnectivity}. \cite{theus2025generalizedlinearmodeconnectivity} connects two independently trained Transformer models by the following procedure: first, they cast both models into a \emph{canonical form} that removes spurious symmetries: LayerNorm scales are absorbed into adjacent weight matrices (replaced by RMSNorm), residual-stream projections are mean-subtracted, and attention weights are reparametrized into per-head QK/OV circuits. The remaining symmetries - permutations of attention heads and MLP neurons, and orthogonal rotations of the residual stream - are resolved by learning alignment matrices via the Hungarian algorithm (for permutations, with straight-through gradients) and SVD projection (for the orthogonal matrix).

When used off-the-shelf, \cite{theus2025generalizedlinearmodeconnectivity} still exhibits a nontrivial loss barrier between the models. Motivated from \citep{garipov2018losssurfacesmodeconnectivity}, we introduce a \emph{polygonal chain} (polychain) interpolation of connecting two models in addition to the canonicalization/alignment procedure of \citep{theus2025generalizedlinearmodeconnectivity}. Let $\theta_A$ and $\theta_B$ be two models. we add a learnable bend-point $\theta_C$ in between which parametrizes the whole curve with two linear segments, 
\begin{equation}\label{eq:polychain}
    \theta(\lambda) = \begin{cases} (1-2\lambda)\,\theta_A + 2\lambda\,\theta_C & \lambda \in [0, 0.5], \\ (2-2\lambda)\,\theta_C + (2\lambda-1)\,\pi(\theta_B) & \lambda \in (0.5, 1], \end{cases}
\end{equation}
where $\pi(\theta_B)$ denotes the learnable aligned model of $\theta_B$. When combined with the same learned symmetric alignment used in \citep{theus2025generalizedlinearmodeconnectivity}, this reduces the barrier on BookCorpus from ${\sim}0.4$ \citep[Table 2]{theus2025generalizedlinearmodeconnectivity} to ${\sim}0.05$ (Table \ref{opt_lmc_tab} below). As the loss barrier diminishes significantly, we use the modified polychain approach as our main algorithm to connect two models. A detailed pseudocode can be found in \cref{observation_appendix}
\begin{table}[ht!]  
\caption{Same-optimizer connectivity loss barrier. We report the validation loss barrier $\max_{\lambda \in [0,1]} \bigl[\mathcal{L}_{\mathrm{val}}(\theta(\lambda)) - \bigl((1{-}\lambda)\mathcal{L}_{\mathrm{val}}(\theta_A) + \lambda\mathcal{L}_{\mathrm{val}}(\theta_B)\bigr)\bigr]$, i.e., the maximum validation loss deviation of the interpolation path above the linear interpolation of endpoint losses. Compared to \citep{theus2025generalizedlinearmodeconnectivity},  the introduction
of polychain path significantly reduces loss barriers in all cases.}
\centering
\setlength{\abovecaptionskip}{10pt}
\resizebox{\linewidth}{!}{%
\begin{tabular}{lccccc} 
\toprule
\makecell[l]{Optimizers \;|\; Datasets} &
 enwik8 & WikiText-103 & Stories & BookCorpus  & One Billion Word \\
\midrule
AdamW (linear) & $0.26\pm 0.05$  & $0.44\pm 0.02$ & $0.35\pm 0.03$ & $0.34\pm 0.00$ & $0.45\pm 0.07$ \\
AdamW (polychain) & $0.04\pm 0.03$ & \bm{$0.05\pm 0.02$} & \bm{$0.05\pm 0.04$} & $0.05\pm 0.01$  &  \bm{$0.06\pm 0.05$} \\
Muon (linear) & $0.18\pm 0.02$ & $0.43\pm 0.02$  & $0.33\pm 0.01$ & $0.32\pm 0.01$ & $0.40\pm 0.01$ \\ 
Muon (polychain) & \bm{$0.01\pm 0.00$} & \bm{$0.05\pm 0.01$}  & \bm{$0.05\pm 0.00$} & \bm{$0.04\pm 0.01$} & \bm{$0.06 \pm 0.00$}  \\
\bottomrule 
\end{tabular}%
}
\label{opt_lmc_tab}
\end{table} 
\subsection{Spectrum Along Connecting Paths}\label{sec::inter_exp}
Empirically, Muon-trained models have a more isotropic singular value spectrum than AdamW-trained ones which exhibit spectral outliers. This matches their different implicit regularization, which makes the spectrum a useful proxy for studying mode connectivity under optimizer-induced constraints.

\begin{figure}[!htbp]
    \centering
    \includegraphics[width=0.85\linewidth]{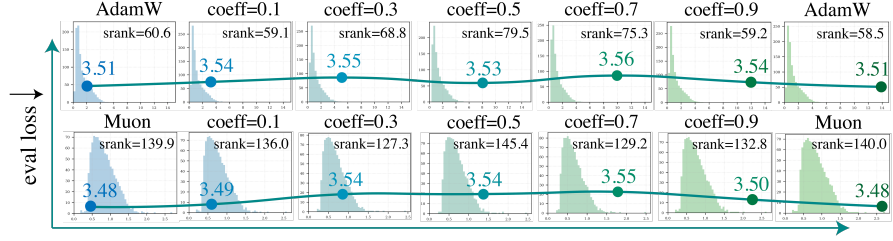}
    \caption{
    Spectrum along same-optimizer connectivity paths.
    Each panel shows the singular value histogram at a different interpolation coefficient. 
    }
    \label{fig:spectral_same_optimizer}
\end{figure}

We first consider connecting two models trained with the same optimizer. In \cref{fig:spectral_same_optimizer}, the endpoints correspond to independently trained models, and we construct a path between them using our modified polychain algorithm. Each panel shows the singular value distribution of the model at a given interpolation coefficient coeff $ \in [0,1]$. Along this path, we track the stable rank number (denoted as srank in our plots) computed as $\sum \sigma_i^2/\sigma_{\max}^2$ where $\sigma_i$ is the $i$th singular value, and observe that it remains largely stable, consistent with the preservation of the overall spectral shape. This result provides empirical support for intra-optimizer mode connectivity in large-scale training: not only do low-loss paths exist, but they remain within a consistent optimizer-induced region of the parameter space.

\begin{figure}[!htbp]
    \centering
    \includegraphics[width=0.85\linewidth]{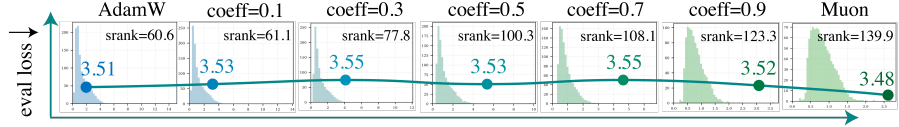}
    \caption{
    Spectrum along the AdamW$\rightarrow$Muon connectivity path.
    Each panel shows the singular value histogram at a different interpolation coefficient.
    }
    \label{fig:spectral_adamw_muon}
\end{figure}

We next connect a model trained with AdamW to one trained with Muon. As shown in \cref{fig:spectral_adamw_muon}, the spectrum evolves smoothly along the path: starting from the AdamW solution with spectral outliers, the singular values gradually become more isotropic, approaching the Muon solution. Such trend  is also quantified by models with monotone increasing stable ranks along the interpolation path. This transition occurs while maintaining relatively stable eval loss. The intermediate models along this path exhibit spectral profiles not typically produced by either AdamW or Muon, suggesting that inter-optimizer connectivity can traverse more diverse regions of the solution space.
\subsection{Out-of-Distribution Generalization Along Cross-Optimizer Paths}
\label{sec:ood}
\begin{figure}[t]
    \centering
    \includegraphics[width=\textwidth]{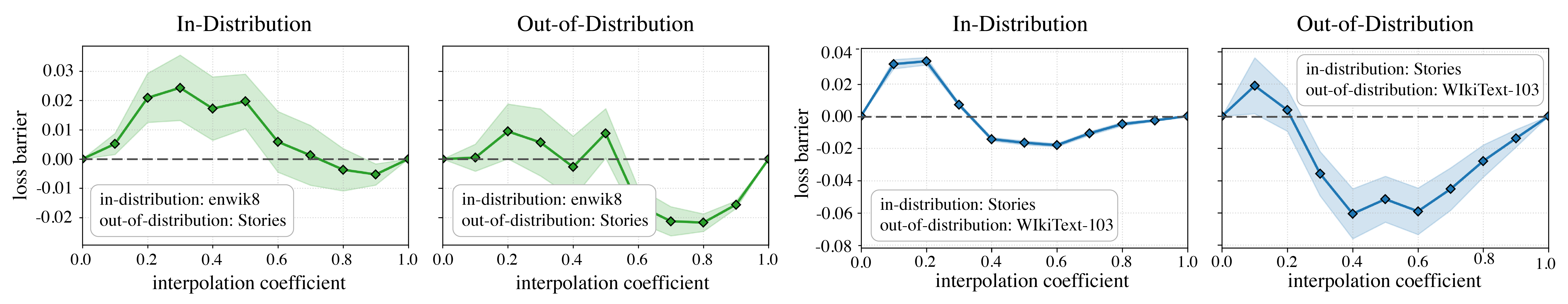}
    \caption{Loss barrier along the AdamW$\to$Muon cross-optimizer path, evaluated in-distribution on enwik8 (left) and out-of-distribution on Stories (right). }
    \label{fig:ood}
\end{figure}

Beyond connectivity itself, we examine whether the spectral diversity along cross-optimizer paths (\cref{sec::inter_exp}) carries practical signal. We take the AdamW$\to$Muon polychain path trained on enwik8 and evaluate it on both enwik8 and Stories, a held-out dataset. Similar procedure is also applied to Stories (in-distribution) and WikiText-103 (out-of-distribution) datasets. \cref{fig:ood} reports the loss $\mathcal{L}_{\mathrm{val}}(\theta(\lambda)) - \bigl((1{-}\lambda)\mathcal{L}_{\mathrm{val}}(\theta_A) + \lambda\mathcal{L}_{\mathrm{val}}(\theta_B)\bigr)$ for both experiments. We observe that the AdamW--Muon interpolation path maintains a small in-distribution loss barrier throughout, while always traversing a regime of negative barrier on out-of-distribution datasets. This suggests that interpolating between solutions found by different optimizers can produce intermediate models with improved out-of-distribution behavior. From a practical perspective, the interpolation path provides a lightweight mechanism for model selection under distribution shift. We present results for more datasets in \cref{sec::ood_appendix}, where similar advantage is persistent across datasets.

\section{Conclusion and Future Work}
\label{sec:conclusion}
We revisit mode connectivity through the lens of optimizer-induced implicit bias. Theoretically, we prove that same-optimizer connectivity holds for sufficiently wide two-layer networks under three classes of norm regularization, and discuss how two different regularized sets are placed in the loss landscape. Empirically, we show that while models trained with either AdamW or Muon can be connected by low-loss paths using our polychain-plus-alignment algorithm, the interpolation path exhibits a characteristic spectral phase transition, and show that it can improve out-of-distribution generalization. 

Several directions remain open. Extending the theory to deep networks and understanding how the spectral transition interacts with network depth would help bridge our two-layer analysis and practical Transformer training. It would also be interesting to characterize optimizer-reachable sets for a broader class of preconditioned methods~\citep{gupta2018shampoopreconditionedstochastictensor,vyas2025soapimprovingstabilizingshampoo} and understand how their implicit biases interact with landscape geometry. Finally, whether one can exploit the spectral diversity along cross-optimizer paths for more effective model merging~\citep{stoica2024zipitmergingmodelsdifferent,crisostomi2024c2m3cycleconsistentmultimodelmerging} or continual learning is a promising practical direction.

\begin{ack}
This work was supported in part by the National Science Foundation (NSF) CAREER Award under Grant CCF-2236829, in part by the National Institutes of Health under Grant 1R01AG08950901A1, in part by the Office of Naval Research under Grant N00014-24-1-2164, and in part by the Defense Advanced Research Projects Agency under Grant HR00112490441.
\end{ack}


\bibliography{neurips_2026}
\bibliographystyle{icml2026}

\newpage
\appendix

\section{Proofs in \cref{sec:theory}}\label{appendix::theory}
\begin{proposition}\label{prop:exact}
Let $f:\mathbb{R}^p \to \mathbb{R}$ be a nonnegative locally Lipschitz function and let $K$ be a norm with dual norm $K_d$. Consider the Lion-$K$ optimizer with update rule
\[
m_{t+1} = \mu\, m_t + g_t,\quad g_t \in \partial_C f(\theta_t),
\qquad
\theta_{t+1} = \theta_t - \eta_t\bigl(v_{t+1} + \lambda \theta_t\bigr),\quad v_{t+1} \in \partial K(m_{t+1}),
\]
where $\partial_C$ denotes the Clarke subdifferential and $\partial K$ the (set-valued) subdifferential of $K$. Suppose $\eta \le \eta_t < 1/\lambda$ for some $\eta > 0$. Define the limit-point set of Lion-$K$ as
\[
\mathcal{O}_{L_K}(\lambda)
:= \arg\min_{\theta \in \mathbb{R}^p} f(\theta)
\;\cap\;
\Bigl\{ \theta \in \mathbb{R}^p \;\big|\; \theta = \lim_{t\to\infty} \theta_t \text{ for some valid Lion-}K \text{ trajectory} \Bigr\}.
\]
Then $\mathcal{O}_{L_K}(\lambda) = \mathcal{O}_K(\lambda)$, i.e., the set of solutions reachable by Lion-$K$ coincides exactly with $\mathcal{O}_K(\lambda)$.
\end{proposition}

\begin{proof}
We first prove $\mathcal{O}_{L_K}(\lambda) \subseteq \mathcal{O}_K(\lambda)$. Take $\theta \in \mathcal{O}_{L_K}(\lambda)$, so $\theta = \lim_{t \to \infty} \theta_t$ for some valid Lion-$K$ trajectory and $f(\theta) = 0$.

For any $v \in \partial K(z)$, since $K$ is a norm, the subgradient inequality $K(w) \ge K(z) + \langle v, w - z \rangle$ implies $K(w) \ge \langle v, w \rangle$ for all $w$. Taking the supremum over $\{w : K(w) \le 1\}$ yields $K_d(v) \le 1$.

Applying this to the update,
\begin{align*}
K_d(\theta_{t+1})
&= K_d\!\bigl((1 - \lambda \eta_t) \theta_t - \eta_t v_{t+1}\bigr) \\
&\le (1 - \lambda \eta_t) K_d(\theta_t) + \eta_t K_d(v_{t+1})
\le (1 - \lambda \eta_t) K_d(\theta_t) + \eta_t.
\end{align*}
Subtracting $1/\lambda$,
\[
K_d(\theta_{t+1}) - \tfrac{1}{\lambda}
\le (1 - \lambda \eta_t)\bigl( K_d(\theta_t) - \tfrac{1}{\lambda} \bigr).
\]
If $K_d(\theta_m) \le 1/\lambda$ for some $m$, then since $\eta_t < 1/\lambda$, we have $K_d(\theta_t) \le 1/\lambda$ for all $t \ge m$, hence $K_d(\theta) \le 1/\lambda$. Otherwise $K_d(\theta_t) > 1/\lambda$ for all $t$, and
\[
0 \le K_d(\theta_t) - \tfrac{1}{\lambda} \le (1 - \eta\lambda)^t \bigl(K_d(\theta_0) - \tfrac{1}{\lambda}\bigr) \to 0,
\]
so $K_d(\theta) = 1/\lambda$. In either case $K_d(\theta) \le 1/\lambda$, so $\theta \in \mathcal{O}_K(\lambda)$.

We now prove $\mathcal{O}_K(\lambda) \subseteq \mathcal{O}_{L_K}(\lambda)$. Take $\theta \in \mathcal{O}_K(\lambda)$, so $f(\theta) = 0$ and $K_d(\theta) \le 1/\lambda$. Since $f \ge 0$ is locally Lipschitz, $\theta$ is a global minimizer and the Clarke optimality condition \citep{clarke1990optimization} gives $0 \in \partial_C f(\theta)$.

We construct a Lion-$K$ trajectory that remains fixed at $\theta$. Initialize $\theta_0 = \theta$ and $m_0 = 0$, and choose $g_0 = 0 \in \partial_C f(\theta_0)$. Then $m_1 = \mu m_0 + g_0 = 0$.

Since $K$ is a norm, $\partial K(0) = \{v : K_d(v) \le 1\}$. Because $K_d(-\lambda \theta) = \lambda K_d(\theta) \le 1$, we have $-\lambda \theta \in \partial K(0)$. Choose $v_1 = -\lambda \theta \in \partial K(m_1)$. Then
\[
\theta_1 = \theta_0 - \eta_0 (v_1 + \lambda \theta_0) = \theta - \eta_0 (-\lambda \theta + \lambda \theta) = \theta.
\]
Repeating these choices at every iteration gives $\theta_t = \theta$ and $m_t = 0$ for all $t$, so $\theta = \lim_{t \to \infty} \theta_t$ is a valid Lion-$K$ limit point. Combined with $f(\theta) = 0$, this gives $\theta \in \mathcal{O}_{L_K}(\lambda)$.
\end{proof}

\begin{proposition}
\label{appxprop2:decoupleweights}
Consider the two-layer ReLU objective
\[
    \mathcal L(W,\alpha)
    =
    \mathcal L\big((XW)_+\alpha,y\big),
\]
and suppose the optimizer is run on the joint variable $(W,\alpha)$ with
decoupled weight decay parameter $\lambda>0$ and constant stepsize
$0<\eta<1/\lambda$.

For AdamW, assume $\beta_1\le \beta_2<1$ and consider the updates
\[
m_{k+1}^W=\beta_1m_k^W+(1-\beta_1)\nabla_W\mathcal L(W_k,\alpha_k),
\quad
v_{k+1}^W=\beta_2v_k^W+(1-\beta_2)(\nabla_W\mathcal L(W_k,\alpha_k))^{\odot 2},
\]
\[
m_{k+1}^\alpha=\beta_1m_k^\alpha+(1-\beta_1)\nabla_\alpha\mathcal L(W_k,\alpha_k),
\quad
v_{k+1}^\alpha=\beta_2v_k^\alpha+(1-\beta_2)(\nabla_\alpha\mathcal L(W_k,\alpha_k))^{\odot 2},
\]
\[
W_{k+1}
=
W_k-\eta\left(
\frac{m_{k+1}^W}{\sqrt{v_{k+1}^W}+\epsilon}
+\lambda W_k
\right),
\quad 
\alpha_{k+1}
=
\alpha_k-\eta\left(
\frac{m_{k+1}^\alpha}{\sqrt{v_{k+1}^\alpha}+\epsilon}
+\lambda \alpha_k
\right).
\]
If the iterates converge to a limit $(W,\alpha)$, then
$(W,\alpha)$ is a KKT point of
\[
    \min_{W,\alpha}\ \mathcal L(W,\alpha)
    \quad
    \mathrm{s.t.}
    \quad
    \max\{\|W\|_{\max},\|\alpha\|_\infty\}
    \le \frac1\lambda .
\]
For Lion-$\mathcal K$, consider the blockwise updates
\[
M_{k+1}^W=\mu M_k^W+\nabla_W\mathcal L(W_k,\alpha_k),\quad M_{k+1}^\alpha=\mu M_k^\alpha+\nabla_\alpha\mathcal L(W_k,\alpha_k),
\]
\[V_{k+1}^W\in \partial K(M_{k+1}^W),\quad V_{k+1}^\alpha\in \partial K(M_{k+1}^\alpha),
\]
\[W_{k+1}=W_k-\eta\left(V_{k+1}^W+\lambda W_k\right),\quad \alpha_{k+1}=\alpha_k-\eta\left(V_{k+1}^\alpha+\lambda\alpha_k\right).
\]
Equivalently, this is Lion-$\mathcal K$ applied to the joint variable$(W,\alpha)$ with the block-separable norm\[K_{\mathrm{joint}}(W,\alpha)=K(W)+K(\alpha).\]

If the iterates converge to a limit $(W,\alpha)$ and $K$ is a norm, then $(W,\alpha)$ is a KKT point of
\[    
    \min_{W,\alpha}\ \mathcal L(W,\alpha)    
    \quad    \mathrm{s.t.}    
    \quad    \max\{K_{d}(W),K_{d}(\alpha)\}    \le \frac1\lambda .
\]    
In particular, for Muon, the KKT constraint becomes\[    \max\{\|W\|_{\mathrm{op}},\|\alpha\|_2\}    \le \frac1\lambda .\]
\end{proposition}
\begin{proof}
    For AdamW, vectorize the joint parameter as
\[
    \theta=(\operatorname{vec}(W),\alpha).
\]
The displayed blockwise AdamW updates are exactly AdamW applied to the
single vector $\theta$, with decoupled weight decay on all coordinates. By
\citet[Theorem~1.1]{xie2024implicitbiasadamwellinfty}, under
$\beta_1\le \beta_2<1$ and $0<\eta<1/\lambda$, any convergent limit point
of AdamW is a KKT point of the constrained problem
\[
    \min_\theta\ \mathcal L(\theta)
    \quad
    \mathrm{s.t.}
    \quad
    \|\theta\|_\infty\le \frac1\lambda .
\]
Rewriting the constraint in terms of $(W,\alpha)$ gives
\[
    \|\theta\|_\infty
    =
    \max\{\|\operatorname{vec}(W)\|_\infty,\|\alpha\|_\infty\}
    =
    \max\{\|W\|_{\max},\|\alpha\|_\infty\}.
\]
Thus $(W,\alpha)$ is a KKT point of
\[
    \min_{W,\alpha}\ \mathcal L(W,\alpha)
    \quad
    \mathrm{s.t.}
    \quad
    \max\{\|W\|_{\max},\|\alpha\|_\infty\}
    \le \frac1\lambda .
\]
For Lion-$\mathcal K$, the blockwise update can be viewed as applying
Lion-$\mathcal K$ to the joint variable $(W,\alpha)$ with the norm
\[
    K_{\mathrm{joint}}(W,\alpha)=K(W)+K(\alpha).
\]
Indeed,
\[
    \partial K_{\mathrm{joint}}(M^W,M^\alpha)
    =
    \partial K(M^W)\times \partial K(M^\alpha),
\]
so the Lion-$\mathcal K$ update on $(W,\alpha)$ with
$K_{\mathrm{joint}}$ is exactly the displayed pair of blockwise updates.

By the implicit-bias characterization of Lion-$\mathcal K$ optimizers
from \citep{chen2025muonoptimizesspectralnorm}, any convergent limit point is a KKT point of
\[
    \min_{W,\alpha}\ \mathcal L(W,\alpha)
    \qquad
    \mathrm{s.t.}
    \qquad
    K_{\mathrm{joint},d}(W,\alpha)\le \frac1\lambda ,
\]
where $K_{\mathrm{joint},d}$ is the dual norm of $K_{\mathrm{joint}}$.

It remains to compute this dual norm. By definition,
\[
\begin{aligned}
K_{\mathrm{joint},d}(U,\beta)
&=
\sup_{K_{\mathrm{joint}}(W,\alpha)\le 1}
\langle U,W\rangle+\langle \beta,\alpha\rangle \\
&=
\sup_{K(W)+K(\alpha)\le 1}
\langle U,W\rangle+\langle \beta,\alpha\rangle .
\end{aligned}
\]
Using the definition of the dual norm,
\[
    \langle U,W\rangle\le K_d(U)K(W),
    \qquad
    \langle \beta,\alpha\rangle\le K_d(\beta)K(\alpha).
\]
Therefore,
\[
\begin{aligned}
K_{\mathrm{joint},d}(U,\beta)
&\le
\sup_{K(W)+K(\alpha)\le 1}
K_d(U)K(W)+K_d(\beta)K(\alpha) \\
&=
\max\{K_d(U),K_d(\beta)\}.
\end{aligned}
\]
The reverse inequality follows by setting either $\alpha=0$ or $W=0$ and
choosing a dual-norm maximizer for the larger block. Hence
\[
    K_{\mathrm{joint},d}(U,\beta)
    =
    \max\{K_d(U),K_d(\beta)\}.
\]
Applying this with $(U,\beta)=(W,\alpha)$ gives the constraint
\[
    \max\{K_d(W),K_d(\alpha)\}\le \frac1\lambda .
\]
Thus $(W,\alpha)$ is a KKT point of the claimed constrained problem.

Finally, for Muon, taking
\[
    K(W)=\|W\|_*,
    \quad
    K(\alpha)=\|\alpha\|_2
\]
gives
\[
    K_d(W)=\|W\|_{\mathrm{op}},
    \quad
    K_d(\alpha)=\|\alpha\|_2.
\]
Therefore the corresponding KKT constraint is
\[
    \max\{\|W\|_{\mathrm{op}},\|\alpha\|_2\}
    \le \frac1\lambda .
\]
This proves the proposition.

\end{proof}
\begin{lemma}
\label{AppxLemma1:smaller} (\cref{Lemma1:smaller} of the paper)
Let $m \ge m^*$ and $R: \|\cdot\|_F, \|\cdot\|_{\max}, \|\cdot\|_{\op}$. Let $S \subseteq \mathcal{O}_R(\lambda)$ satisfies $(W, \alpha) \in S \implies \sum_{i=1}^m \mathbf{1}(W_i \neq 0) \le m^*/2$. For all $(W, \alpha) \in S$ and any two points $A \ne B \in S$, there exists a continuous path in $\mathcal{O}_R(\lambda)$ that connects $A$ and $B$.
\end{lemma}
\begin{proof}
First we show that any two permutations of a solution in $S$ is connected by a continuous path in $\mathcal{O}_R(\lambda)$. For $(W, \alpha) \in S$, we know the existence of a zero slot where $W_z = 0$ because $m \geq m^{*} > m^{*}/2$. We show two things:
\begin{enumerate}
    \item Assume $W_i \ne 0$ and $W_j = 0$. We have a path from $(W, \alpha)$ to $(\tilde{W}, \tilde{\alpha})$, where $\tilde{W}$ and $\tilde{\alpha}$ has everything identical except the $i, j$-th coordinates are swapped. We do this by transformation:
    \[ 
    W_k(t) = \begin{cases} 
    W_k \text{ for } k \ne i, j \\[4pt] 
    \sqrt{1-t}\, W_i \text{ for } k = i \\[4pt] 
    \sqrt{t}\, W_i \text{ for } k = j \end{cases} 
    \alpha_k(t) = \begin{cases} \alpha_k \text{ for } k \ne i, j \\[4pt]
    \sqrt{1-t}\, \alpha_i \text{ for } k = i \\[4pt] 
    \sqrt{t}\, \alpha_i \text{ for } k = j \end{cases} \text{ for } t \in [0, 1].
    \]
    It is clear that $(XW(t))_+ \alpha(t) = y$ for $t \in [0, 1]$, $W(0)=W, \alpha(0)=\alpha, W(1)=\tilde{W}$ and $\alpha(1)=\tilde{\alpha}$. The regularization constraint is satisfied as $\|W(t)\|_F$ and $\|\alpha(t)\|_2$ are preserved during the transformation, as well as the Gram matrix
    \[
    \sum_{i=1}^{m} W_i(t) W_i(t)^{T}.
    \]
    For elementwise $\ell_\infty$ norm, certain columns are scaled by less than 1, so $\|W(t)\|_{\max} \le \|W\|_{\max}$ and $\|\alpha(t)\|_\infty \le \|\alpha\|_\infty$.
    \item Assume $W_i \ne 0$ and $W_j \ne 0$. We have a path where weights are swapped. We use $W_z = 0$. First swap $(i, z)$. Then $(i, j)$. Then $(z, j)$. (We are swapping a zero and nonzero at each swap).
\end{enumerate}
Hence you can connect any permutations that are nonzero or one of them is nonzero, and when you use operation (1) to match the support of the permutation that you want to move to, then use operation (2) to match the exact permutation, you have moved from one permutation to another.

Now take $A \ne B \in S$. Let's connect $A$ with $\pi(A)$, $B$ with $\pi(B)$, where $\pi(A)$ has nonzero weights for $W_1^A, \dots, W_{m^*/2}^A$ and $B$ has nonzero weights for $W_{m^*/2+1}^B, \dots, W_{m^*}^B$ so that the indices where $\pi(A)$ has nonzero first-layer weights and the indices where $\pi(B)$ has nonzero first-layer weights are disjoint. Let's denote the first and second layers of $\pi(A)$, $\pi(B)$ as $W^{A}, \alpha^{A}, W^{B}, \alpha^{B}$. We connect $\pi(A)$ and $\pi(B)$ with $W(t)$:
\[ 
W_k(t) = \begin{cases} 
\sqrt{1-t}\, W_k^A \text{ for } k \le m^*/2 \\[4pt]
\sqrt{t}\, W_k^B \text{ for } m^*/2 < k \le m^* \\[4pt]
0 \text{ otherwise} \end{cases} 
\alpha_k(t) = \begin{cases} \sqrt{1-t}\, \alpha_k^A \text{ for } k \le m^*/2 \\[4pt] 
\sqrt{t}\, \alpha_k^B \text{ for } m^*/2 < k \le m^* \\[4pt]
0 \text{ otherwise} \end{cases},
\]
for $t \in [0,1]$. We know that $(XW(t))_+ \alpha(t) = (1-t)(XW^A)_+ \alpha^A + t(XW^B)_+ \alpha^B = y$.
Now we check the norm condition for three norms. We know that Frobenius norm $\|W(t)\|_F \le 1/\lambda$ because the Frobenius norm becomes 
\[
\| W(t) \|_F^2 = \sum_{k=1}^{m^{*}/2} (1-t)\| W_k^{A}\|_2^2 + \sum_{k = m^{*}/2+1}^{m^{*}} t \|W_k^{B} \|_2^2 = (1-t) \| W^A\|_F^2 + t\|W^B\|_F^2 \leq \frac{1}{\lambda^2},
\]
and same holds for $\alpha(t)$. 
For operator norm we know that the Gram matrix becomes
\[
W(t)W(t)^{T} = \sum_{k=1}^{m^{*}/2} (1-t) W_k^A(W_k^A)^{T} + \sum_{k=m^{*}/2+1}^{m^{*}} t W_k^{B}(W_k^{B})^{T} = (1-t)W^{A}(W^{A})^{T} + tW^{B}(W^{B})^{T},
\]
and $\lambda_{\max}(W(t)W(t)^{T}) \leq (1-t) \lambda_{\max}(W^A(W^A)^{T}) + t\lambda_{\max}(W^B(W^B)^{T})$ because $W(t)W(t)^{T}$, $W^A(W^A)^{T}$, $W^B(W^B)^T$ are all Hermitian. Hence 
therefore $\|W(t)\|_{\op}^2 \le (1-t)\|W^A\|_{\op}^2 + t\|W^B\|_{\op}^2 \leq 1/\lambda^2$. For $\| \cdot \|_{\max}$, the weights scale with nonnegative number no larger than 1, so the weight constraint is satisfied. At last, we know $W(0) = W^{A}, \alpha(0) = \alpha^{A}, W(1) = W^{B}, \alpha(1) = \alpha^{B}$.

Hence we know that we can connect $A$ with $\pi(A)$, $B$ with $\pi(B)$, and $\pi(A)$ with $\pi(B)$ for any two $A \ne B \in S$ with a continuous curve in $\mathcal{O}_R(\lambda)$, given $R: \|\cdot\|_F, \|\cdot\|_{\max}, \|\cdot\|_{\op}$.
\end{proof}

\begin{proposition}
\label{Appxprop3:reduction} (\cref{prop3:reduction} of the paper)
Let $R: \|\cdot\|_F$ or $\|\cdot\|_{\op}$. Then $\forall A \in \mathcal{O}_R(\lambda)$, there exists a continuous path from $A$ to $\mathcal{O}_R^{\nm}(\lambda)$ in $\mathcal{O}_R(\lambda)$.
\end{proposition}

\begin{proof}
If $A \in \mathcal{O}_R^{\nm}(\lambda)$ we do not have to do anything. If $A \notin \mathcal{O}_R^{\nm}(\lambda)$, we repeat the merging process until we cannot merge anymore. If we can't merge, we shrink $\alpha_i$ to 0 if $W_i = 0$, and $W_i$ to 0 if $\alpha_i = 0$. Either case $(XW)_+ \alpha$ is preserved, and the constraint is satisfied. That is because we can think of the shrinking process as $W(t) = WD(t)$ where $D(t)$ is a diagonal matrix that has $(1-t)$ on the entry you want to shrink to 0, and $\|D(t)\|_{\op} \leq 1$. Hence we can make $W_i \alpha_i = 0 \Rightarrow W_i = 0$ and $\alpha_i = 0$.

Now we will illustrate how merging is possible. Assume $W_i \alpha_i \ne 0, W_j \alpha_j \ne 0, \mathbf{1}(XW_i \ge 0) = \mathbf{1}(XW_j \ge 0), \mathbf{1}(\alpha_i \ge 0) = \mathbf{1}(\alpha_j \ge 0)$. We think of the curve
\[ 
    W_k(t) = \begin{cases} 
    W_k \text{ for } k \ne i, j \\[4pt] 
    \sqrt{1-t}\, W_i \text{ for } k = i \\[4pt] 
    \dfrac{tW_i|\alpha_i|+W_j|\alpha_j|}{\sqrt{\alpha_j^2 + t\alpha_i^2}}\text{ for } k = j \end{cases} 
    \alpha_k(t) = \begin{cases} \alpha_k \text{ for } k \ne i, j \\[4pt]
    \sqrt{1-t}\, \alpha_i \text{ for } k = i \\[4pt] 
    \sqrt{\alpha_j^2 + t\alpha_i^2}\, \textrm{sign}(\alpha_i) \text{ for } k = j \end{cases} \text{ for } t \in [0, 1].
\]

For $W_i(t), W_j(t)$ we have
\begin{align*}
(XW_i(t))_{+}\alpha_i(t) + (XW_j(t))_{+}\alpha_j(t) &= (1-t)(XW_i)_{+}\alpha_i + t(XW_i)_{+}\alpha_i + (XW_j)_{+}\alpha_j\\
&= (XW_i)_{+}\alpha_i + (XW_j)_{+}\alpha_j,
\end{align*}
hence $(XW(t))_{+}\alpha(t) = y$ is preserved as a constant. 

Now we deal with the regularization. For the Frobenius norm, we know that
\begin{align*}
\|W(t)\|_F^2 = \sum_{k=1}^{m} \|W_k(t)\|_2^2 
&= \sum_{k\neq i,j} \|W_k(t)\|_2^2 + (1-t)\|W_i\|_2^2 + \left\|\frac{tW_i|\alpha_i|+W_j|\alpha_j|}{\sqrt{\alpha_j^2 + t\alpha_i^2}}\right\|_2^2 \\
&\leq \sum_{k \neq i,j} \|W_k(t)\|_2^2 + (1-t)\|W_i\|_2^2 + t\|W_i\|_2^2 + \|W_j\|_2^2\\
&= \|W\|_F^2.
\end{align*}
The inequality follows from
\[
\|aX+bY\|_2^2 \leq (|a|\|X\|_2 + |b|\|Y\|_2)^2 \leq (a^2+b^2)(\|X\|_2^2+\|Y\|_2^2),
\]
and letting $a = \sqrt{t}|\alpha_i|/\sqrt{t\alpha_i^2+\alpha_j^2}, b = |\alpha_j|/\sqrt{t\alpha_i^2+\alpha_j^2}$, $X = \sqrt{t}W_i$, $Y = W_j$. Also $\|\alpha(t)\|_2$ is preserved as a constant.

For the operator norm, we show that for any $x \in \mathbb{R}^{d}$, $x^{T}W(t)W(t)^{T}x \leq x^{T}WW^{T}x$ holds. That is because 
\begin{align*}
x^{T}W(t)W(t)^{T}x &= \sum_{k \neq i,j}x^{T}W_kW_k^{T}x + (1-t)(x^{T}W_i)^2 + \frac{(tx^{T}W_i|\alpha_i| + x^{T}W_j|\alpha_j|)^2}{t\alpha_i^2 + \alpha_j^2}\\
&\leq \sum_{k \neq i,j}x^{T}W_kW_k^{T}x + (1-t)(x^{T}W_i)^2 + (t(x^{T}W_i)^2 + (x^{T}W_j)^2)\\
&= \sum_{k=1}^{m} x^{T}W_kW_k^{T}x = x^{T}WW^{T}x.
\end{align*}
The inequality follows from Cauchy-Schwartz. The inequality implies $\lambda_{\max} (W(t)W(t)^{T}) \leq \lambda_{\max}(WW^{T})$ and $\| W(t)\|_{\op} \leq \|W\|_{\op} \leq 1/\lambda$.

Hence we arrive at a point where the two weights $W_i$ and $W_j$ have been merged. Repeat until we cannot merge (process terminates as $m$ is finite). Then we use the process in the beginning to reduce $W_i\alpha_i = 0$ but $W_i \neq 0$ or $\alpha_i \neq 0$.
\end{proof}

\begin{theorem}
\label{t1:intramode_rot} (\cref{t1:intramode_rot_31} of the paper)
Intra-mode connectivity holds for Normalized-GD with momentum and Muon when $m \ge 4P$. Here $P$ is the number of possible hyperplane arrangement patterns $\mathbf{1}(Xh \ge 0)$.
\end{theorem}
\begin{proof}
Suppose $m \ge 4P$ and $R: \|\cdot\|_F, \|\cdot\|_{\op}$. First, from \cref{AppxLemma1:smaller}, when we set $S = \mathcal{O}_R^{\nm}(\lambda)$, we know that for $(W, \alpha) \in S, \sum \mathbf{1}(W_i \ne 0) \le 2P$ and hence for two points $A, B \in S$, $\exists$ path in $\mathcal{O}_R(\lambda)$ that connects these two. Moreover, we know that given two points $A, B \in \mathcal{O}_R(\lambda)$, there exists $A', B' \in \mathcal{O}_R^{\nm}(\lambda)$ such that $A - A'$ and $B - B'$ are connected in $\mathcal{O}_R(\lambda)$. Connect $A'$ and $B'$ to find a path that connects $A$ and $B$ in $\mathcal{O}_R(\lambda)$. Hence $\mathcal{O}_R(\lambda)$ is connected for $R: \|\cdot\|_F, \|\cdot\|_{\op}$ and by definition, Normalized-GD with momentum and Muon is intra-mode connected.
\end{proof}

\begin{proposition}
\label{Prop5:Dickson}
Let $Z = \{ (t, s) \mid P_{(t, s)} \ne \emptyset \}$ and let $Z_A$ be the set of minimal elements of $Z$, where minimal elements mean an element $\mathfrak{m} \in Z$ s.t. $s \le \mathfrak{m}$ for $s \in Z$ implies $\mathfrak{m} \le s$ ($\le$ is the partial order). Then $Z_A$ is finite and for all $p \in Z$, there exists $\mathfrak{m} \in Z_A$ s.t. $\mathfrak{m} \le p$.
\end{proposition}

\begin{proof}
The proof resorts to Dickson's Lemma \cite{dickson1913finiteness}. The case where $Z = \emptyset$ is trivial. Let $Z \neq \emptyset$. First, we know that $Z$ is a subset of $\mathbb{Z}_{\ge 0}^{2P}$, hence $Z_A$ is finite by Dickson's Lemma. Let $Z' = \{ r \mid r \le p \}$ for any given $p \in Z$. By Dickson's Lemma, $Z'$ has a minimal element $\mathfrak{m}_p$. This implies for any $s \in Z$ s.t. $s \le \mathfrak{m}_p, \mathfrak{m}_p \le s$ and $\mathfrak{m}_p \in Z_A$.
\end{proof}

\begin{proposition}
\label{prop4:connectPandTh}
Let $m \ge \sum_{i=1}^P (t_i + s_i)$. If $P_{(t, s)} \ne \emptyset$ and $(t,s) \in Z_A$, then $\mathcal{O}_{\|\cdot\|_{\max}, (t, s)}^{\eq}(\lambda) \ne \emptyset$.
\end{proposition}

\begin{proof}
Take $(u_i, v_i)_{i=1}^P \in P_{(t, s)}$. We set the weights of $(W, \alpha)$ as following:

Initialize $k_0 = 1$. $i = 1, 2, ... P$. For each $i$,

if $t_i \ne 0$ then $u_i \neq 0$ because $(t,s) \in Z_A$, put $(W_{k_0} = \frac{u_i \lambda}{t_i}, \alpha_{k_0} = 1/\lambda), \dots, (W_{k_0+t_i-1} = \frac{u_i \lambda}{t_i}, \alpha_{k_0+t_i-1} = 1/\lambda)$. Then increment $k_0$ to $k_0 + t_i$.

If $s_i \ne 0$ then $v_i \neq 0$ because $(t,s) \in Z_A$, put $(W_{k_0} = \frac{v_i \lambda}{s_i}, \alpha_{k_0} = -1/\lambda), \dots, (W_{k_0+s_i-1} = \frac{v_i \lambda}{s_i}, \alpha_{k_0+s_i-1} = -1/\lambda)$. Then increment $k_0$ to $k_0 + s_i$.

Leave the rest $k > k_0$ first/second layer weights to 0.

Is the constructed $(W, \alpha)$ in $\mathcal{O}_{\|\cdot\|_{\max}, (t, s)}^{\eq}(\lambda)$? First, it is in $\mathcal{O}_{\|\cdot\|_{\max}}(\lambda)$ because:
$(XW)_+ \alpha = \sum D_i X (\frac{u_i \lambda}{t_i} \cdot \frac{1}{\lambda}) t_i - \sum D_i X (\frac{v_i \lambda}{s_i} \cdot \frac{1}{\lambda}) s_i = \sum D_i X (u_i - v_i) = y$.
$\|W\|_{\max} = \max \{ \frac{1}{t_i} \|u_i \lambda\|_\infty, \frac{1}{s_i} \|v_i \lambda\|_\infty \} \le 1/\lambda$, $\|\alpha\|_\infty = 1/\lambda$.
Also it is in $\mathcal{O}_{\|\cdot\|_{\max}}^{\eq}(\lambda)$ because $\alpha = \pm 1/\lambda$, $W_i \alpha_i = 0 \Rightarrow W_i=\alpha_i=0$ and solutions with same pattern are identical. At last, by construction $\sum_i \mathbf{1}(\mathbf{1}(XW_i \geq 0) = D_p)\mathbf{1}(\alpha_i > 0) = t_p$, and same for $s_p$.
\end{proof}

\begin{proposition}
\label{AppxProp6:l1reduction} (\cref{Prop6:l1reduction} of the paper)
Let $S = \bigcup_{(t,s) \in Z_A} \mathcal{O}_{\|\cdot\|_{\max}, (t,s)}^{\eq}(\lambda)$ and the critical width
\[
m^* = 2 \max_{(t,s) \in Z_A} \sum_{i=1}^P t_i + s_i.
\] Then the followings hold:
\begin{enumerate}
    \item[i)] Let $m \geq m^*$. If $(W, \alpha) \in S$, $\sum_{i=1}^m \mathbf{1}(W_i \neq 0) \leq m^*/2$.
    \item[ii)] For all $A \in \mathcal{O}_{\|\cdot\|_{\max}}(\lambda)$, there exists continuous path from $A$ to $S$ in $\mathcal{O}_{\|\cdot\|_{\max}}(\lambda)$.
\end{enumerate}
\end{proposition}
\begin{proof}
i) If $(W, \alpha) \in S$, $(W, \alpha) \in \mathcal{O}_{\|\cdot\|_{\max}, [t',s']}^{\eq}(\lambda)$ for some $[t',s'] \in Z_A$. Let's count the number of nonzero first-layer columns. This is going to be $\sum_{i=1}^P t'_i +s'_i$, because $\sum_{i=1}^{P} t'_i$ will count the number of nonzero neurons with $\alpha_i > 0$, and $\sum_{i=1}^{P} s'_i$ will count the number of nonzero neurons with $\alpha_i < 0$.
Hence $\sum_{i=1}^m \mathbf{1}(W_i \neq 0) = \sum_{i=1}^P (t'_i + s'_i) \leq m^*/2$. The last inequality holds because $[t',s'] \in Z_A$.

ii) Take a point $A \in \mathcal{O}_{\|\cdot\|_{\max}}(\lambda)$. We will construct a continuous path from $A$ to a point in $S$ with first finding a point $A_1 \in \mathcal{O}_{\|\cdot\|_{\max}}^{\eq}(\lambda)$, then connecting $A_1$ to a point in $S$.
We follow the following process: first, given $\{W_1, \dots, W_m\}, \{\alpha_1, \dots, \alpha_m\}$, if $W_i \alpha_i = 0$, we can shrink the other without violating the regularization. Hence, we can move from $A$ to $A'$ such that the nonzero weights of $A$ and $A'$ are same, but for $A'$, inactive neurons are removed.
Now we think of the procedure where $|\alpha_i| < \frac{1}{\lambda}$ we scale $\alpha_i$ and $W_i$ at the same time. Specifically:
\[\alpha_i(t) = \alpha_i + (\tfrac{1}{\lambda} - \alpha_i)t, \quad W_i(t) = \frac{W_i \alpha_i}{\alpha_i + (\tfrac{1}{\lambda} - \alpha_i)t} \quad \text{if } \alpha_i > 0,\]
\[\alpha_i(t) = \alpha_i + (-\tfrac{1}{\lambda} - \alpha_i)t, \quad W_i(t) = \frac{W_i |\alpha_i|}{|\alpha_i + (-\tfrac{1}{\lambda} - \alpha_i)t|} \quad \text{if } \alpha_i < 0.\]
Both cases, $W_i(t)$ is scaled by scale $ < 1$, so the constraint is not violated. Also, $(XW)_{+}\alpha = y$ due to positive homogeneity.

Hence we have made $\alpha_i$ either $\pm \frac{1}{\lambda}$ if it is not 0.
Now for $W_i$'s that have same arrangement pattern and second-layer weight, without loss of generality say $W_1, W_2, \dots, W_k$, we can move from $(W_1, \dots, W_k)$ to $(\frac{\sum W_i}{k}, \dots, \frac{\sum W_i}{k})$ by considering $\Delta(t)$:
\[
\Delta(0) = I_{k \times k}, \quad \Delta(1) = \frac{1}{k} J_{k \times k}, \quad \Delta(t) = (1-t)\Delta(0) + t\Delta(1)
\] and $W(t)[1:k] = [W_1|W_2|\cdots|W_k]\Delta(t)$. During the interpolation we have each column $b$ a convex combination of the $k$ weights because the sum of each column entries is 1, and the norm constraint is preserved. The fit is also preserved because the sum of each row is 1 for $\Delta(t)$. After repeating this finitely, we can make all $\mathbf{1}(XW_j \geq 0) = \mathbf{1}(XW_j \geq 0)$ and $\alpha_i = \alpha_j \rightarrow W_i = W_j$.

Hence we continuously moved from $A \in \mathcal{O}_{\|\cdot\|_{\max}}(\lambda)$ to $A_1 \in \mathcal{O}_{\|\cdot\|_{\max}}^{\eq}(\lambda)$. The next step finding $[t_1, s_1] \in \mathbb{Z}_{\geq 0}^{2P}$ such that $A_1 \in \mathcal{O}_{\|\cdot\|_{\max}, [t_1, s_1]}^{\eq}(\lambda)$. Such $[t_1, s_1]$ can be easily found by letting
\[
t_{1p} = \sum_{i=1}^{m} \mathbf{1}(\mathbf{1}(XW^{A_1}_i \geq 0) = D_p)\mathbf{1}(\alpha^{A_1}_i > 0), \quad s_{1p} = \sum_{i=1}^{m} \mathbf{1}(\mathbf{1}(XW^{A_1}_i \geq 0) = D_p)\mathbf{1}(\alpha^{A_1}_i < 0),
\]
where $A_1 = (W^{A_1}, \alpha^{A_1})$.
For that $[t_1, s_1]$, we see that $P_{[t_1, s_1]} \neq \emptyset$. That is because we can construct $(u_i, v_i)_{i=1}^P$ as:
\begin{align*}
u_i = t_i W^{A_1}_j / \lambda \ \  \textrm{if}\ \  \alpha^{A_1}_j > 0,\ \mathbf{1}(XW^{A_1}_j \geq 0) = D_i,\\
v_i = s_i W^{A_1}_j / \lambda \ \  \textrm{if} \ \  \alpha^{A_1}_j < 0,\ \mathbf{1}(XW^{A_1}_j \geq 0) = D_i.
\end{align*}
As $\|W^{A_1}_j\|_\infty \leq \frac{1}{\lambda}$, $\|u_i W^{A_1}_j / \lambda\|_\infty \leq \frac{t_i}{\lambda}$, similarly $\|s_i W^{A_1}_j / \lambda\|_\infty \leq s_i / \lambda$. Also if $u_i \neq 0$, $u_i = (t_i/\lambda) \cdot W^{A_1}_j$, which is a positive multiple of $W^{A_1}_j$, hence $\mathbf{1}(Xu_i \geq 0) = D_i$, same for $v_i$.

At last, $\sum_{i=1}^m (XW^{A_1}_i)_{+} \alpha^{A_1}_i$ is preserved: for $\alpha^{A_1}_i > 0, \mathbf{1}(XW^{A_1}_i \geq 0) = D_1$, take them as $W_1, \dots, W_{t_{11}}$, then
$\sum_{i=1}^{t_{11}} (XW^{A_1}_i)_{+} \alpha_i^{A_1} = \sum_{i=1}^{t_{11}} D_1 X W_i^{A_1} \frac{1}{\lambda} = D_1 X \frac{t_{11} W_i}{\lambda}$, hence it is $D_1Xu_1$. Same for $v_i$ but we have a negation. This means $\sum_{i=1}^P D_i X (u_i - v_i) = (XW^{A_1})_{+}\alpha^{A_1} = y$.

Thus $P_{[t_1,s_1]} \neq \emptyset$ and $[t_1, s_1] \in Z$. By \cref{Prop5:Dickson} we know the existence of $[t_m, s_m] \in Z_A$ such that $[t_m, s_m] \leq [t_1, s_1]$. 

From \cref{prop4:connectPandTh}, we can find $A_1' = (W^{A_1'}, \alpha^{A_1'}) \in \mathcal{O}_{\|\cdot\|_{\max}, [t_m, s_m]}^{\eq}(\lambda)$ because $m \geq m^* \geq \sum_{i=1}^{P} t_{mi} + s_{mi}$. We can connect $(W^{A_1'},\alpha^{A_1'})$ and $(W^{A_1},\alpha^{A_1})$ by more interpolation.
For each $i$-th element in $S$ we have $t_{1i}$ identical $(W, 1/\lambda)$'s for $A_1$, and we have $t_{mi}$ $(W', 1/\lambda)$'s for $A_1'$ (which is the point in $\mathcal{O}_{\|\cdot\|_{\max}, [t_m, s_m]}^{\eq}(\lambda)$). We do
$W_j(t) = t W_j^{A_1} + (1-t) W_j^{A_1'}$ for $j = 1, 2, \dots, t_{mi}$, $W_j(t) = (1-t)W_j^{A_1'}$ for $j > t_{mi}, j \leq t_{1i}$.
We do this type of interpolation for all layers. 

The constraint is satisfied because we either do convex combination or shrinkage. Now the model fit is preserved, because $(XW)_{+}\alpha = \sum_{i=1}^P D_i X (\frac{t_i W_i}{\lambda}) - D_i X (\frac{s_i W'_i}{\lambda})$ (here, we know that for same arrangement pattern we have same vector, so we take the representative as $W_i, W'_i$ for $A_1$),
$(XW)_{+}\alpha = \sum_{i=1}^P D_i X (\frac{t_{1i} z_i}{\lambda}) - D_i X (\frac{s_{1i} z'_i}{\lambda})$ for $A_1'$ similarly, and at time $t$ we have
\begin{align*}
&\sum_{i=1}^P \Bigg[
    D_i X \frac{t_{1i}(1-t)W_i + t z_i}{\lambda}
    + D_i X \frac{(t_{1i}-t_{mi})(1-t)W_i}{\lambda} \\
&\qquad\quad
    - D_i X \frac{s_{mi}(1-t)W'_i + t z'_i}{\lambda}
    - D_i X \frac{(s_{1i}-s_{mi})(1-t)W'_i}{\lambda}
\Bigg] \\
&= \sum_{i=1}^P \left(
    \frac{t_{mi} D_i X z_i}{\lambda}
    + (1-t)\frac{t_{1i} D_i X W_i}{\lambda}
    - t\frac{s_{mi} D_i X z'_i}{\lambda}
    - (1-t)\frac{s_{1i} D_i X W'_i}{\lambda}
\right) \\
&= t \sum_{i=1}^P D_i X
    \left(\frac{t_{mi} z_i}{\lambda} - \frac{s_{mi} z'_i}{\lambda}\right)
    + (1-t) \sum_{i=1}^P D_i X
    \left(\frac{t_{1i} W_i}{\lambda} - \frac{s_{1i} W'_i}{\lambda}\right) \\
&= y.
\end{align*}
Hence $(XW)_{+}\alpha$ is preserved throughout, and the curve is in $\mathcal{O}_{\|\cdot\|_{\max}}(\lambda)$. Thus I moved continuously from $A$ to a point in $\mathcal{O}_{\|\cdot\|_{\max}, [t_m, s_m]}^{\eq}(\lambda)$ where $[t_m, s_m] \in Z_A$. This means I moved to a point in $S$ using only the points in $\mathcal{O}_{\|\cdot\|_{\max}}(\lambda)$.
\end{proof}

\begin{theorem}   
\label{Appxt2:intra_signum}
(Intra-mode connectivity for Signum) Let $m \geq m^*$ given as $m^* = 2 \max_{(t,s) \in Z_A} \sum_{i=1}^P t_i+s_i$, where $Z_A$ is defined in \cref{Prop5:Dickson}. Then $\mathcal{O}_{\| \cdot \|_{\max}}(\lambda)$ is connected, and $m^*$ is determined only with $(X, y, \lambda)$.
\end{theorem}
\begin{proof}
Suppose $m \geq m^*$. First, from \cref{AppxLemma1:smaller}, when we set $S = \bigcup_{(t,s) \in Z_A} \mathcal{O}_{\|\cdot\|_{\max}, [t,s]}^{\eq}(\lambda)$, we know that for $(W, \alpha) \in S$, $\sum_{i=1}^m \mathbf{1}(W_i \neq 0) \leq m^*/2$ from \cref{AppxProp6:l1reduction}(i) and hence for two points $A, B \in S, \exists$ path in $\mathcal{O}_{\|\cdot\|_{\max}}(\lambda)$ that connects $A$ and $B$. Moreover, also from \cref{AppxProp6:l1reduction}(ii), we know for any $A \in \mathcal{O}_{\|\cdot\|_{\max}}(\lambda)$, there exists a path from $A$ to $S$ in $\mathcal{O}_{\|\cdot\|_{\max}}(\lambda)$. For any two $A, B \in \mathcal{O}_{\|\cdot\|_{\max}}(\lambda)$, we can move $A \rightarrow A', B \rightarrow B'$ where $A', B' \in S$, then connect $A'$ and $B'$ in $\mathcal{O}_{\|\cdot\|_{\max}}(\lambda)$. This shows that $\mathcal{O}_{\|\cdot\|_{\max}}(\lambda)$ is connected when $m \geq m^*$.
\end{proof}

\begin{lemma}
\label{Appxlem2: regcrit_fit}
(Critical regularization for perfect fitting) Let $m$ be given and $R: \|\cdot\|_F, \|\cdot\|_{\max}, \|\cdot\|_{\op}$. Assume there exists $(W_0, \alpha_0)$ with width $m_0 \le m$ such that $(XW_0)_+\alpha_0 = y$, and let $\lambda_0 := 1/\max\{R(W_0), R(\alpha_0)\}$. Then $\mathcal{O}_R(m, \lambda) \neq \emptyset \iff \lambda \leq \lambda_f^*(m)$ where $\lambda_f^*(m)$ is defined as 
\[
\lambda_f^{*}(m) = \frac{1}{\min_{(XW)_{+}\alpha=y} \max\{R(W), R(\alpha)\}}.
\]
\end{lemma}
\begin{proof}
We first show $\lambda^{*}_f(m)$ is well defined, i.e., that $\max\{R(W), R(\alpha)\}$ attains a minimum on the set $F = \{(W,\alpha)\ | \ (XW)_{+}\alpha=y\}$.

By assumption, $(W_0, \alpha_0) \in F$ with $\max\{R(W_0), R(\alpha_0)\} = 1/\lambda_0$. Setting $C := 1/\lambda_0$, the set
\[
F \cap \{(W,\alpha) \ | \ R(W) \leq C, \ R(\alpha) \leq C\}
\]
is nonempty (contains $(W_0, \alpha_0)$) and compact (closed by continuity of the constraints, bounded by $C$). Hence $\max\{R(W), R(\alpha)\}$ attains a minimum $\mu \leq C$ on this restricted set; outside the restricted set, $\max\{R(W), R(\alpha)\} > C \geq \mu$. So $\mu$ is the minimum on all of $F$. Also, $\mu > 0$ because $y \neq 0$.

Now $\lambda^{*}_f(m) = 1/\mu$. If $\lambda \leq \lambda_f^{*}(m)$, there exists $(W, \alpha) \in F$ achieving the minimum, so $\max\{R(W), R(\alpha)\} = 1/\lambda^{*}_f(m) \leq 1/\lambda$ and $\mathcal{O}_R(m,\lambda) \neq \emptyset$. If $\lambda > \lambda_f^{*}(m)$ and $\mathcal{O}_R(m, \lambda) \neq \emptyset$, then there exists $(W,\alpha) \in F$ with $\max\{R(W), R(\alpha)\} \leq 1/\lambda < \mu$, contradicting that $\mu$ is the minimum of $\max\{R(W), R(\alpha)\}$ on $F$.
\end{proof}

\begin{proposition}
\label{appxprop7:regime_FOP}
($\|\cdot\|_F$ and $\|\cdot\|_{\op}$: $(m, \lambda)$ of interest) Let $R: \|\cdot\|_F$ or $\|\cdot\|_{\op}$. Let $m \geq \max\{m_R^{o}, 4P\}$ and $0 <\lambda \leq \lambda_f^*(m)$. Then the set $\mathcal{O}_R(m, \lambda)$ is connected and nonempty.
\end{proposition}
\begin{proof}
We plot the region where $\mathcal{O}_R(m, \lambda)$ is nonempty and connected separately.
For connected region, \cref{t1:intramode_rot} shows that for all $\lambda > 0$, $m \geq 4P$ implies connectivity.
For nonempty region, we can see that from \cref{Appxlem2: regcrit_fit}, iff $m \geq m_R^*$ and $\lambda \leq \lambda_f^*(m)$, $\mathcal{O}_R(m, \lambda)$ is nonempty. We know that as $\mathcal{O}_R(m_R^o, \lambda_R^o) \neq \emptyset$, we can see $\mathcal{O}_R(m_R^o, \lambda_f^*(m_R^o)) \neq \emptyset$ for $m \geq m_R^o$ and can apply \cref{Appxlem2: regcrit_fit} for $m \geq m_R^o$.
Intersect the two regions. if $m \geq \max\{4P, m_R^o\}$, $0 < \lambda \leq \lambda_f^*(m)$, $\mathcal{O}_R(m, \lambda)$ is nonempty and connected.
\end{proof}

\begin{proposition}
\label{AppxProp8:Signum_reg}
Let $R: \|\cdot\|_{\max}$. For $m \geq 4P+1$, there exists critical regularization $\lambda_c^*(m)$ for each $m$ such that $\lambda \leq \lambda_c^{*}(m)$ implies $\mathcal{O}_R(m, \lambda)$ is connected.
\end{proposition}
\begin{proof} 
By \cref{Appxt2:intra_signum}, if $m \geq 2 \max_{(t,s) \in Z_A(\lambda)} \sum_{i=1}^P (t_i+s_i)$, $\mathcal{O}_R(m, \lambda)$ is connected. Let's define $m^{*}(\lambda)$ as
\[
m^{*}(\lambda) = \begin{cases}
2 \max_{(t,s) \in Z_A(\lambda)} \sum_{i=1}^P (t_i+s_i) \quad \textrm{if} \quad Z_A(\lambda) \neq \emptyset \\[4pt]
0 \quad \textrm{if} \quad Z_A(\lambda) = \emptyset.
\end{cases}
\]
We wish to find a tractable upper bound of $m^{*}(\lambda)$ that is increasing so that we can find critical width $M^{*}$ and $\lambda^{*}_c(m)$. Specifically, we will show that $m^{*}(\lambda) \leq 4P + 4\lambda^2 M P$ for a constant $M$ only dependent on $(X, y)$. 

First, let
\begin{align*}
\mathcal{P} = \Big\{ (u_i, v_i, a_i, b_i)_{i=1}^{P} \ | \ \sum_{i=1}^{P} &D_iX(u_i-v_i) = y, \ u_i = 0 \ \textrm{or} \ \mathbf{1}(Xu_i \geq 0) = D_i,\\
\ &v_i = 0 \ \textrm{or} \ \mathbf{1}(Xv_i \geq 0) = D_i, \ \|u_i\|_\infty \leq a_i,\ \|v_i\|_\infty \leq b_i \Big\}.
\end{align*}
Then we know that $\mathcal{P}$ is a union of at most $2^{2P}$ polyhedron, where for each polyhedron is for each chosen support. Specifically, let $S_1, S_2 \subseteq [P]$ is chosen. $\mathcal{P}_{(S_1,S_2)}$ is
\begin{align*}
\mathcal{P}_{(S_1,S_2)} = \Big\{ (u_i, v_i, a_i, b_i)_{i=1}^{P} \ | \ \sum_{i=1}^{P} &D_iX(u_i-v_i) = y, \ \mathbf{1}(Xu_i \geq 0) = D_i,\ \mathbf{1}(Xv_i \geq 0) = D_i, \ \\
&\|u_i\|_\infty \leq a_i,\ \|v_i\|_\infty \leq b_i,
u_i = 0 \ \textrm{if}\ i \in S_1, v_i = 0 \ \textrm{if}\ i \in S_2 \Big\},
\end{align*}
and 
\[
\mathcal{P} = \cup_{S_1, S_2 \subseteq [P]} \mathcal{P}_{(S_1,S_2)}.
\]
Note that $\mathcal{P}_{(S_1, S_2)}$ possibly has strict inequalities. For each polyhedron $\mathcal{P}_{(S_1,S_2)}$, the projection of the closure
\[
\mathcal{Q}_{(S_1,S_2)} = \Big\{(a_i, b_i)_{i=1}^{P} \ | \ (u_i, v_i, a_i, b_i)_{i=1}^{P} \in \mathcal{\bar{P}}_{(S_1,S_2)} \Big\}
\]
is also a polyhedron. Now, let the minimal points of $\mathcal{Q}_{(S_1,S_2)}$ be $\mathcal{Q}_m(S_1,S_2)$. Then $\mathcal{Q}_m(S_1,S_2)$ is bounded. The argument follows from Minkowski-Weyl theorem: because $\mathcal{Q}_{(S_1,S_2)}$ is a polyhedron, there exists finite vectors $p_1, p_2, \cdots p_N$, $r_1, r_2, \cdots , r_{N'}$ such that 
\[
\mathcal{Q}_{(S_1,S_2)} = \textrm{conv}(p_1, p_2, \cdots, p_N) + \textrm{nonneg}(r_1, r_2, \cdots, r_{N'}).
\]
We first know that $\textrm{conv}(p_1, p_2, \cdots, p_N)$ is bounded. Also, $\mathcal{Q}_{(S_1,S_2)} \subset \mathbb{R}_{\geq 0}^{2P}$. Hence each $r_i \geq 0$: if one element of $r_i < 0$, then increasing that direction to infinity will create a negative element. Now for a point $p \in \mathcal{Q}_m$, when we write
\[
p = \tau + \sum_{i=1}^{N'} c_i r_i, \quad c_i \geq 0
\]
for $\tau \in \textrm{conv}(p_1, p_2, \cdots p_N)$, then all $c_i = 0$ because if not, it is not minimal. Hence $\mathcal{Q}_m(S_1,S_2) \subseteq \textrm{conv}(p_1, p_2, \cdots , p_N)$, and $\mathcal{Q}_m(S_1,S_2)$ is bounded. 

Next, we show that for the real projection that we are interested in,
\[
\Pi_{(S_1,S_2)} = \Big\{(a_i, b_i)_{i=1}^{P} \ | \ (u_i, v_i, a_i, b_i)_{i=1}^{P} \in \mathcal{P}_{(S_1,S_2)} \Big\},
\]
and its minimal set $\Pi_{m}(S_1,S_2)$, that $\Pi_{m}(S_1,S_2)$ is bounded. That is because
\[
\Pi_{m}(S_1,S_2) \subseteq \mathcal{Q}_m(S_1,S_2).
\]
The proof is as follows: take $p \in \Pi_m(S_1,S_2)$, and let $(u^{p}, v^{p}, p) \in \mathcal{P}_{(S_1,S_2)}$. Suppose $p \notin \mathcal{Q}_m(S_1,S_2)$. There exists $q \in \mathcal{Q}_{(S_1,S_2)}$ such that $q < p$, and there exists $(u^c, v^c, q) \in \mathcal{\bar{P}}_{(S_1,S_2)}$. Hence, there exists a sequence $(u_k, v_k, q_k) \in \mathcal{P}_{(S_1,S_2)}$ that converges to $(u^c, v^c, q)$. For this sequence there exists $N$ such that $q_N < p$ which is a contradiction that $p \in \Pi_m(S_1,S_2)$ as $q_N \in \Pi(S_1,S_2)$. Hence $p \in \mathcal{Q}_m(S_1,S_2)$, $\Pi_{m}(S_1,S_2)$ is bounded. 

Then we see that the projected set
\[
\Pi = \Big\{(a_i, b_i)_{i=1}^{P} \ | \ (u_i, v_i, a_i, b_i)_{i=1}^{P} \in \mathcal{P} \Big\} = \cup_{S_1,S_2 \subseteq [P]} \Pi_{(S_1,S_2)},
\]
and its minimal set $\Pi_m$, then $\Pi_m$ is bounded. That is because 
\[
\Pi_m \subseteq \cup_{S_1,S_2 \subseteq [P]} \Pi_m(S_1,S_2),
\] 
because for any $p \in \Pi_m$, $p \in \Pi$ and $p \in \Pi_{(S_1,S_2)}$ for some $S_1, S_2$ and as $p \in \Pi_m$, $p \in \Pi_m(S_1,S_2)$. As $\Pi_m(S_1,S_2)$ is bounded, its finite union is bounded, and $\Pi_m$ is also bounded. 
Say $\Pi_m \subseteq [0, M]^{2P}$. As $\mathcal{P}$ is a function of $(X, y)$, $M$ is also a function of $(X, y)$.

Now take $(t, s) \in Z_A(\lambda)$. Then there exists $(u_i, v_i)_{i=1}^{P} \in P_{(t,s)}(\lambda)$ because $\mathcal{P}_{(t,s)} \neq \emptyset$. For that $(u_i, v_i)$, we know that 
\[
\sum_{i=1}^{P} D_iX(u_i - v_i) = y,
\]
\[
u_i = 0 \ \  \textrm{or} \ \ \mathbf{1}(Xu_i \geq 0) = D_i, \quad v_i = 0 \ \  \textrm{or} \ \ \mathbf{1}(Xv_i \geq 0) = D_i,
\]
\[
\| u_i \|_\infty \leq \frac{t_i}{\lambda^2}, \quad \| v_i \|_\infty \leq \frac{s_i}{\lambda^2}, \quad i \in [P].
\]
For this point $(u_i, v_i)_{i=1}^{P}$, we know that $(\|u_i\|_\infty, \|v_i\|_\infty)_{i=1}^{P} \in \Pi$. Think of the point 
\[
(\gamma_i, \delta_i)_{i=1}^{P} \in \Pi_m \ \ \textrm{such that} \ \ (\gamma_i, \delta_i)_{i=1}^{P} \leq (\|u_i\|_\infty, \|v_i\|_\infty)_{i=1}^{P}.
\]
If $\gamma_k \leq (t_k - 1)/\lambda^2$ or $\delta_k \leq (s_k - 1)/\lambda^2$ for some $k \in [P]$, as $(\gamma_i, \delta_i)_{i=1}^{P} \in \Pi$, there exists $(u_i', v_i')_{i=1}^{P}$ such that 
\[
\sum_{i=1}^{P} D_iX(u_i' - v_i') = y,
\]
\[
u_i' = 0 \ \  \textrm{or} \ \ \mathbf{1}(Xu_i' \geq 0) = D_i, \quad v_i' = 0 \ \  \textrm{or} \ \ \mathbf{1}(Xv_i' \geq 0) = D_i,
\]
\[
\| u_i' \|_\infty \leq \frac{t_i'}{\lambda^2}, \quad \| v_i'\|_\infty \leq \frac{s_i'}{\lambda^2}, \quad i \in [P].
\]
where $(t',s') \neq (t,s)$ because $\gamma_k \leq (t_k - 1)/\lambda^2$ or $\delta_k \leq (s_k - 1)/\lambda^2$ for some $k \in [P]$. Hence for all $k$, 
\[
\gamma_k \in \left[\frac{t_k-1}{\lambda^2}, \frac{t_k}{\lambda^2}\right], \quad \delta_k \in \left[\frac{s_k-1}{\lambda^2}, \frac{s_k}{\lambda^2}\right],
\]
and 
\[
\prod_{k=1}^{P} \left[\frac{t_k-1}{\lambda^2}, \frac{t_k}{\lambda^2}\right] \times \prod_{k=1}^{P} \left[\frac{s_k-1}{\lambda^2}, \frac{s_k}{\lambda^2}\right] \cap \Pi_m \neq \emptyset.
\]
As $\Pi_m \subseteq [0,M]^{2P}$, we know that for each $t_i$ and $s_i$ 
\[
(t_i-1) \leq \lambda^2M, \quad (s_i-1)\leq \lambda^2M.
\]
Hence for $(t,s) \in Z_A(\lambda)$, 
\[
\sum_{i=1}^{P} t_i + s_i \leq 2P + 2P\lambda^2M,
\]
and $m^{*}(\lambda) \leq 4P + 4\lambda^2MP$. This means given $\lambda$, if $m \geq 4P + 4\lambda^2MP$, $\mathcal{O}_R(m, \lambda)$ is connected. In other words, for given $m$ such that $m \geq 4P+1$, when we define 
\[
\lambda_c^{*}(m) = \sqrt{\frac{1}{M}\Big(\frac{m}{4P} - 1\Big)},
\]
if $\lambda \leq \lambda_c^{*}(m) = \sqrt{\frac{1}{M}\Big(\frac{m}{4P} - 1\Big)}$ we have $m \geq 4P + 4\lambda^2MP$ and $\mathcal{O}_R(m, \lambda)$ is connected. This finishes the proof.
\end{proof}

\begin{proposition}
\label{prop9:nontrivial_l1}
($R: \|\cdot\|_{\max}$ norm: $(m, \lambda)$ of interest) Let $R: \|\cdot\|_{\max}, m \geq \max \{m_R^o, 4P+1\}$ and $0 < \lambda \leq \min\{\lambda_f^*(m), \lambda_c^{*}(m)\}$. Then the set $\mathcal{O}_R(m, \lambda)$ is connected and nonempty.
\end{proposition}
\begin{proof} 
From \cref{AppxProp8:Signum_reg} we know that when $m \geq 4P+1$, $\mathcal{O}_R(m,\lambda)$ is connected when $0 < \lambda \leq \lambda^{*}_c(m)$.
Also, if $\lambda \leq \lambda_f^*(m)$ and $m \geq m_R^o$, we can see that for all $m$, there is $\lambda$ such that $\mathcal{O}_R(m, \lambda)$ is nonempty, and we can use \cref{Appxlem2: regcrit_fit}.
Intersect the two regions.
\end{proof}

\begin{theorem}
\label{appxthm:inter-connectivity} (\cref{thm:inter-connectivity} of the paper)
Let $R_1, R_2$ be two regularizers in $\{\|\cdot\|_{\max}, \|\cdot\|_F, \|\cdot\|_{\op}\}$, $m \geq \max\{m_{R_1}^*, m_{R_2}^*\}$, and $0 < \lambda_1 \leq \lambda_{R_1}^*(m)$ fixed. As $\lambda_2$ varies from $0$ to $\lambda_{R_2}^{*}(m)$:
\begin{enumerate}
    \item[i)] If $\mathcal{O}_{R_1}(\lambda_1) \cap \mathcal{O}_{R_2}(\lambda_{R_2}^{*}(m)) \neq \emptyset$, the union $\mathcal{O}_{R_1}(\lambda_1) \cup \mathcal{O}_{R_2}(\lambda_2)$ has a single connected component for all $\lambda_2$.
    \item[ii)] If $\mathcal{O}_{R_1}(\lambda_1) \cap \mathcal{O}_{R_2}(\lambda_{R_2}^{*}(m)) = \emptyset$, there exists $\lambda_2^*(\lambda_1)$ such that the union has one component when $\lambda_2 \leq \lambda_2^*(\lambda_1)$ and two components when $\lambda_2 > \lambda_2^*(\lambda_1)$.
\end{enumerate}
\end{theorem}

\begin{proof}
i) Assume $\mathcal{O}_{R_1}(\lambda_1) \cap \mathcal{O}_{R_2}(\lambda_{R_2}^*(m)) \neq \emptyset$. Then, for $\lambda_2 \leq \lambda_{R_2}^*(m)$, we know that $\mathcal{O}_{R_2}(\lambda_{R_2}^*(m)) \subseteq \mathcal{O}_{R_2}(\lambda_2)$. Hence, $\mathcal{O}_{R_2}(\lambda_2) \cap \mathcal{O}_{R_1}(\lambda_1) \neq \emptyset$ for $\lambda_2 \leq \lambda_{R_2}^*(m)$ and as $\mathcal{O}_{R_1}(\lambda_1), \mathcal{O}_{R_2}(\lambda_2)$ is connected, $\mathcal{O}_{R_1}(\lambda_1) \cup \mathcal{O}_{R_2}(\lambda_2)$ is connected.

ii) Assume $\mathcal{O}_{R_1}(\lambda_1) \cap \mathcal{O}_{R_2}(\lambda_{R_2}^*(m)) = \emptyset$. Now define
\[
\lambda_2^*(\lambda_1) = \inf_{\lambda_2 > 0} \{ \lambda_2 \mid \mathcal{O}_{R_1}(\lambda_1) \cap \mathcal{O}_{R_2}(\lambda_2) = \emptyset \}.
\]
We know that $\lambda_{R_2}^*(m) \in \{ \lambda_2 \mid \mathcal{O}_{R_1}(\lambda_1) \cap \mathcal{O}_{R_2}(\lambda_2) = \emptyset \}$, hence $\lambda_2^*$ is well defined. 

When $\lambda_2 > \lambda_2^*(\lambda_1)$, say $\mathcal{O}_{R_1}(\lambda_1) \cap \mathcal{O}_{R_2}(\lambda_2) \neq \emptyset.$ Then $ \mathcal{O}_{R_1}(\lambda_1) \cap \mathcal{O}_{R_2}(t) \neq \emptyset$ for $t \leq \lambda_2$. This means $\{ t \mid \mathcal{O}_{R_1}(\lambda_1) \cap \mathcal{O}_{R_2}(t) = \emptyset \} \subseteq (\lambda_2, \infty)$ and $\lambda^{*}_2(\lambda_1)\geq \lambda_2$. This is a contradiction, so 
\[\lambda_2 > \lambda_2^*(\lambda_1) \Rightarrow \mathcal{O}_{R_1}(\lambda_1) \cap \mathcal{O}_{R_2}(\lambda_2) = \emptyset.
\] 
If $\mathcal{O}_{R_1}(\lambda_1) \cap \mathcal{O}_{R_2}(\lambda_2) = \emptyset$ we can say the set union is not connected since both $\mathcal{O}_{R_1}(\lambda_1)$ and $\mathcal{O}_{R_2}(\lambda_2)$ are compact.

Now we show that if $\lambda_2 = \lambda_2^*(\lambda_1)$, the set $\mathcal{O}_{R_1}(\lambda_1) \cap \mathcal{O}_{R_2}(\lambda_2^*(\lambda_1)) \neq \emptyset$. If $\mathcal{O}_{R_1}(\lambda_1) \cap \mathcal{O}_{R_2}(\lambda_2^*(\lambda_1)) = \emptyset$, we can find $\mathcal{O}_{R_2}(\lambda_2^*(\lambda_1) - \varepsilon) \cap \mathcal{O}_{R_1}(\lambda_1) = \emptyset$ for some $\varepsilon > 0$. The method is clear: think of 
\[f(W, \alpha) = \max \{ R_2(W), R_2(\alpha) \}.
\] 
Think of the function's image on the set $\mathbb{R}^{d \times m} \times \mathbb{R}^m = \mathcal{O}_{R_1}(\lambda_1)$. We know that $\mathcal{O}_{R_1}(\lambda_1)$ is compact, hence $f$ has a minimum. Let the minimum be $f_m$. As $\mathcal{O}_{R_1}(\lambda_1) \cap \mathcal{O}_{R_2}(\lambda_2^*(\lambda_1)) = \emptyset$, $\frac{1}{\lambda_2^*(\lambda_1)} < f_m$ must hold. Then, take $\lambda' = \frac{2}{f_m + 1/\lambda_2^*(\lambda_1)} < \lambda_2^*(\lambda_1)$ to see that $1/\lambda' < f_m$, and $\mathcal{O}_{R_1}(\lambda_1) \cap \mathcal{O}_{R_2}(\lambda') = \emptyset.$ This implies $\lambda' \in \{\lambda \mid \mathcal{O}_{R_1}(\lambda_1) \cap \mathcal{O}_{R_2}(\lambda_2) = \emptyset\}$, contradiction that $\lambda_2^*(\lambda_1)$ is a lower bound.
This means $\mathcal{O}_{R_1}(\lambda_1) \cap \mathcal{O}_{R_2}(\lambda_2^*(\lambda_1)) \neq \emptyset$ and is nonempty for $\lambda_2 \leq \lambda_2^*(\lambda_1)$. This finishes the proof.
\end{proof}

\begin{proposition}
(\cref{prop2:connectivity} of the paper)
Let
$\mathcal O=\{(W,\alpha)\mid (XW)_+\alpha=y\}$
and let \(\mathcal C_1,\ldots,\mathcal C_N\) be its connected components. For two norms \(R_1,R_2\), write
\[
\mathcal O_{R_\ell}(\lambda_\ell)
=
\{(W,\alpha)\in\mathcal O\mid R_\ell(W,\alpha)\le 1/\lambda_\ell\},
\qquad \ell=1,2.
\]
Suppose that for some \(i\neq j\),
$
\mathcal O_{R_1}(\lambda_1)\subseteq \mathcal C_i,
$
and
$
\mathcal O_{R_2}(\lambda_2)\subseteq \mathcal C_j.
$
Then the following hold:
\begin{itemize}[left=0pt]
    \item For any \(\ell \in \{1,2\}\) and any \(p,q \in \mathcal O_{R_\ell}(\lambda_\ell)\), there exists a continuous path \(\gamma:[0,1]\to \mathcal O\) such that \(\gamma(0)=p\) and \(\gamma(1)=q\).
    \item There is no continuous path \(\gamma:[0,1]\to \mathcal O\) such that
    $ \gamma(0)\in \mathcal O_{R_1}(\lambda_1),
    \ \gamma(1)\in \mathcal O_{R_2}(\lambda_2).$
\end{itemize}
\end{proposition}
\begin{proof}
Since \(\mathcal O_{R_1}(\lambda_1)\subseteq \mathcal C_i\) and
\(\mathcal O_{R_2}(\lambda_2)\subseteq \mathcal C_j\), any two points in the same
regularized set lie in the same connected component of \(\mathcal O\). Since the
components of \(\mathcal O\) are path-connected, for any
\(\ell\in\{1,2\}\) and any \(p,q\in \mathcal O_{R_\ell}(\lambda_\ell)\), there exists
a continuous path \(\gamma:[0,1]\to \mathcal O\) such that
\[
\gamma(0)=p,\qquad \gamma(1)=q.
\]
Now suppose, for contradiction, that there exists a continuous path
\(\gamma:[0,1]\to \mathcal O\) such that
\[
\gamma(0)\in \mathcal O_{R_1}(\lambda_1),
\qquad
\gamma(1)\in \mathcal O_{R_2}(\lambda_2).
\]
Then
\[
\gamma(0)\in \mathcal C_i,
\qquad
\gamma(1)\in \mathcal C_j.
\]
However, \(\gamma([0,1])\) is connected and contained in \(\mathcal O\). Hence it must
lie entirely in a single connected component of \(\mathcal O\). This contradicts
\(i\neq j\). Therefore no such path exists.
\end{proof}

\begin{proposition} Let $A\in\mathbb{R}^{d\times d}$ be invertible,
\[
X=\begin{bmatrix}A\\-A\end{bmatrix}\in\mathbb{R}^{2d\times d},
\qquad
y\in\mathbb{R}^{2d}_{>0},
\]
and consider width $m=2$. For the point $(W,\alpha) \in \mathcal{O}$ for the solution set
\[
\mathcal{O}=\{(W,\alpha)\mid (XW)_+\alpha=y\},
\]
we have
\[
A_{i\cdot}W_{\cdot j} \neq 0\ \  \forall i \in [d], j \in [2].
\]
\label{prop1:barrier}
\end{proposition}
\begin{proof}
The proof follows by looking at the nonzero elements of $(Xu)_{+}$ for $u \in \mathbb{R}^{d}$. We know that if $A_{i\cdot}u> 0$ for some $i$, then $-A_{i\cdot}u < 0$ and both cannot simultaneously be positive. Hence, in $(Xu)_{+}$ there can be at most $d$ nonzero elements, where the position of nonzero elements are determined by the tuple of signs 
\[
(A_{1\cdot}u, A_{2\cdot}u, \cdots, A_{d \cdot}u).
\]
Now we go to the solution set $\mathcal{O}$. As we have
\[
(XW)_{+}\alpha = (XW_{\cdot 1})_{+}\alpha_1 + (XW_{\cdot 2})_{+}\alpha_2 = y,
\]
and $y$ does not have nonzero elements, each $(XW_{\cdot 1})_{+}$, $(XW_{\cdot 2})_{+}$ should have exactly $d$ nonzero elements. If $A_{i\cdot}W_{\cdot j} = 0$ for some $i \in [d], j \in [2]$ this is impossible because for that $i$, both $(A_{i\cdot}W_{\cdot j})_{+} = (-A_{i\cdot}W_{\cdot j})_{+} = 0$ and $(XW_{\cdot j})_{+}$ will have nonzero elements strictly less than $d$.
\end{proof}

\begin{proposition}
\label{appxp2:connectedcomponentchar}
Let $A\in\mathbb R^{d\times d}$ be invertible and set
\[
X=\begin{bmatrix}A\\-A\end{bmatrix},
\]
$y\in\mathbb R^{2d}_{>0}$. For
$\sigma\in(\pm1)^d$, denote $y_\sigma\in\mathbb R^d$ by
\[
(y_\sigma)_i
=
\begin{cases}
y_i, & \sigma_i=+1,\\
-y_{d+i}, & \sigma_i=-1.
\end{cases}
\]
Also, let
\[
\mathcal C_{\sigma\mid-\sigma}
=
\left\{
\left(
\begin{bmatrix}
A^{-1}\dfrac{y_\sigma}{\alpha_1}
&
A^{-1}\dfrac{y_{-\sigma}}{\alpha_2}
\end{bmatrix},
(\alpha_1,\alpha_2)
\right)
:\alpha_1,\alpha_2>0
\right\}.
\]
For the solution set $\mathcal{O}$, we have 
\[
\mathcal{O} = \bigsqcup_{\sigma \in \{\pm 1\}^{d}}
\mathcal{C}_{\sigma \mid -\sigma}.
\]
More importantly, the solution set $\mathcal{O}$ has $2^{d}$ connected components, and has $2^{d-1}$ connected components up to permutation.
\end{proposition}
\begin{proof}
First let's prove that 
\[
\bigsqcup_{\sigma \in \{\pm 1\}^{d}}
\mathcal{C}_{\sigma \mid -\sigma} \subseteq \mathcal{O}.
\]
Take $(W, \alpha) \in \mathcal{C}_{\sigma | -\sigma}$. We know that
\[
(\begin{bmatrix}
    A \\ -A
\end{bmatrix} W_{\cdot 1})_{+}\alpha_1 = (\begin{bmatrix}
    A \\ -A
\end{bmatrix} A^{-1} y_{\sigma})_{+} = (\begin{bmatrix}
    y_{\sigma} \\ -y_\sigma
\end{bmatrix})_{+}
\]
Now the entries of $(\begin{bmatrix}
    y_{\sigma} \\ -y_\sigma
\end{bmatrix})_{+}$ are:
if $\sigma_i = 1$ then the $i$-th entry of $(\begin{bmatrix}
    y_{\sigma} \\ -y_\sigma
\end{bmatrix})_{+}$ is $y_i$ the $i+d$-th entry is 0, if $\sigma_i = -1$ then the $i$-th entry is 0, the $i+d$-th entry is $y_{d+i}$.
Similarly, we have 
\[
(\begin{bmatrix}
    A \\ -A
\end{bmatrix} W_{\cdot 2})_{+}\alpha_2 = (\begin{bmatrix}
    y_{-\sigma} \\ -y_{-\sigma}
\end{bmatrix})_{+}
\]
and if $\sigma_i = 1$ then the $i$-th entry is 0, the $i+d$-th entry is $y_{i+d}$, and if $\sigma_i = -1$ the $i$-th entry is $y_i$, the $i+d$-th entry is 0. Hence we have
\[
(\begin{bmatrix}
    y_{\sigma} \\ -y_{\sigma}
\end{bmatrix})_{+}+(\begin{bmatrix}
    y_{-\sigma} \\ -y_{-\sigma}
\end{bmatrix})_{+} = y.
\]
Now we prove that 
\[
\mathcal{O} \subseteq \bigsqcup_{\sigma \in \{\pm 1\}^{d}}
\mathcal{C}_{\sigma \mid -\sigma}.
\]
Take an arbitrary $(W,\alpha)\in\mathcal O$, and write
\[
W=\begin{bmatrix}W_{\cdot 1} & W_{\cdot 2}\end{bmatrix},
\qquad
\alpha=(\alpha_1,\alpha_2),
\]
with $\alpha_1,\alpha_2>0$. We know that $\alpha_1, \alpha_2 > 0$ because they are the ratio between two positive vectors. Define
\[
z_1:=AW_{\cdot 1},\qquad z_2:=AW_{\cdot 2}.
\]
Then the equations $(XW)_+\alpha=y$ are equivalent, for each
$i=1,\dots,d$, to
\[
\alpha_1 (z_{1i})_+ + \alpha_2 (z_{2i})_{+} = y_i,
\]
and
\[
\alpha_1 (-z_{1i})_+ + \alpha_2 (-z_{2i})_+ = y_{d+i}.
\]
Since $y_i>0$ and $y_{d+i}>0$, both the positive and negative sides must be
activated. Hence $(z_1)_i$ and $(z_2)_i$ must have opposite signs for every
$i$. In particular, neither can vanish. Define
\[
\sigma_i:=\operatorname{sign}(z_{1i})\in\{\pm1\}.
\]
Then necessarily
\[
\operatorname{sign}(z_{2i})=-\sigma_i.
\]
If $\sigma_i=+1$, then $z_{1i}>0$ and $z_{2i}<0$, so the two equations give
\[
\alpha_1z_{1i}=y_i,
\qquad
-\alpha_2z_{2i}=y_{d+i}.
\]
If $\sigma_i=-1$, then $z_{1i}<0$ and $z_{2i}>0$, so
\[
-\alpha_1z_{1i}=y_{d+i},
\qquad
\alpha_2z_{2i}=y_i.
\]
Equivalently,
\[
z_1=\frac{y_\sigma}{\alpha_1},
\qquad
z_2=\frac{y_{-\sigma}}{\alpha_2}.
\]
Since $A$ is invertible, this implies
\[
W_{\cdot 1}=A^{-1}\frac{y_\sigma}{\alpha_1},
\qquad
W_{\cdot 2}=A^{-1}\frac{y_{-\sigma}}{\alpha_2}.
\]
Therefore
\[
(W,\alpha)\in \mathcal C_{\sigma\mid-\sigma}.
\]
This proves
\[
\mathcal{O} \subseteq
\bigsqcup_{\sigma\in\{\pm1\}^d}\mathcal C_{\sigma\mid-\sigma}.
\]

It remains to justify that the sets $\mathcal{C}_{\sigma | -\sigma}$ are disjoint and each set is connected. The sign vector $\sigma$ is uniquely determined by
\[
\sigma=\operatorname{sign}(AW_{\cdot 1}),
\]
so two distinct sign vectors give disjoint sets. Hence the above union is
indeed disjoint.

For each fixed $\sigma$, the set $\mathcal C_{\sigma\mid-\sigma}$ is connected,
because it is the image of the connected set $(0,\infty)^2$ under the continuous
map
\[
(\alpha_1,\alpha_2)\mapsto
\left(
\begin{bmatrix}
A^{-1}\dfrac{y_\sigma}{\alpha_1}
&
A^{-1}\dfrac{y_{-\sigma}}{\alpha_2}
\end{bmatrix},
(\alpha_1,\alpha_2)
\right).
\]
Moreover, no continuous path in $\mathcal O$ can move from
$\mathcal C_{\sigma\mid-\sigma}$ to
$\mathcal C_{\tau\mid-\tau}$ with $\sigma\neq\tau$, because along such a path
the sign vector
\[
\operatorname{sign}(AW_{\cdot 1})
\]
would have to change. Hence, for some coordinate $i$, one would have
\[
(AW_{\cdot 1})_i=0
\]
at some intermediate point. But this is impossible in $\mathcal O$ due to \cref{prop1:barrier}.
Therefore the sets $\mathcal C_{\sigma\mid-\sigma}$ are precisely the connected
components of $\mathcal O$. Since there are $2^d$ choices of
$\sigma\in\{\pm1\}^d$, the solution set has $2^d$ connected components.
Finally, permuting the two hidden neurons maps
$\mathcal C_{\sigma\mid-\sigma}$ onto
$\mathcal C_{-\sigma\mid\sigma}$.
Thus the $2^d$ components are paired by the equivalence
\[
\sigma\sim -\sigma,
\]
and therefore the number of connected components modulo permutation is
\[
\frac{2^d}{2}=2^{d-1}.
\]
\end{proof}
\begin{proposition}
Let $p,q\in\mathbb R^d$ and define
\[
\mathcal{C}(p,q)
=
\left\{
\left(
\begin{bmatrix}
p/\alpha_1 & q/\alpha_2
\end{bmatrix},
(\alpha_1,\alpha_2)
\right)
:\ \alpha_1,\alpha_2>0
\right\}.
\]
Then
\[
R_{\infty}(\mathcal{C}(p,q))
=
\max\{\|p\|_\infty,\|q\|_\infty\}^{1/2},
\]
and
\[
R_{\mathrm{op}}(\mathcal{C}(p,q))
=
\bigl(\|p\|_2^2+\|q\|_2^2+2|p^\top q|\bigr)^{1/4} = \max \{\|p+q\|_2, \|p-q\|_2\}^{1/2}.
\]
\end{proposition}

\begin{proof}
Write
\[
W(\alpha_1,\alpha_2)
=
\begin{bmatrix}
p/\alpha_1 & q/\alpha_2
\end{bmatrix}.
\]
For the $\ell_\infty$ objective, we have
\[
\max\{\|W(\alpha_1,\alpha_2)\|_\infty,\|\alpha\|_\infty\}
=
\max\left\{
\frac{\|p\|_\infty}{\alpha_1},
\frac{\|q\|_\infty}{\alpha_2},
\alpha_1,
\alpha_2
\right\}.
\]
Optimizing over $\alpha_1,\alpha_2>0$ gives
\[
R_{\infty}(\mathcal{C}(p,q))
=
\max\{\|p\|_\infty^{1/2},\|q\|_\infty^{1/2}\}.
\]
We now consider the operator norm objective. Write
\[
\alpha_1=\rho\sqrt{s},
\qquad
\alpha_2=\rho\sqrt{1-s},
\qquad
0<s<1,\quad \rho>0.
\]
Then $\|\alpha\|_2=\rho$ and
\[
W(\alpha_1,\alpha_2)
=
\frac1{\rho}
\begin{bmatrix}
p/\sqrt{s} & q/\sqrt{1-s}
\end{bmatrix}.
\]
Thus, for fixed $s$,
\[
\max\{\|W\|_{\mathrm{op}},\|\alpha\|_2\}
=
\max\left\{
\frac1{\rho}
\left\|
\begin{bmatrix}
p/\sqrt{s} & q/\sqrt{1-s}
\end{bmatrix}
\right\|_{\mathrm{op}},
\rho
\right\}.
\]
Optimizing over $\rho>0$ balances the two terms, giving
\[
R_{\mathrm{op}}(\mathcal{C}(p,q))
=
\inf_{0<s<1}
\left\|
\begin{bmatrix}
p/\sqrt{s} & q/\sqrt{1-s}
\end{bmatrix}
\right\|_{\mathrm{op}}^{1/2}.
\]
Equivalently,
\[
R_{\mathrm{op}}(\mathcal{C}(p,q))
=
\inf_{0<s<1}
\lambda_{\max}\!\left(
\frac{pp^\top}{s}
+
\frac{qq^\top}{1-s}
\right)^{1/4}.
\]

Let
\[
a:=\|p\|_2^2,\qquad
b:=\|q\|_2^2,\qquad
c:=p^\top q,
\]
and choose $\tau\in\{\pm1\}$ such that $\tau c=|c|$. For every $s\in(0,1)$,
\[
\frac{pp^\top}{s}
+
\frac{qq^\top}{1-s}
-
(p+\tau q)(p+\tau q)^\top
=
\left(
\sqrt{\frac{1-s}{s}}p
-
\tau\sqrt{\frac{s}{1-s}}q
\right)
\left(
\sqrt{\frac{1-s}{s}}p
-
\tau\sqrt{\frac{s}{1-s}}q
\right)^\top
\succeq 0.
\]
Therefore
\[
\lambda_{\max}\!\left(
\frac{pp^\top}{s}
+
\frac{qq^\top}{1-s}
\right)
\ge
\|(p+\tau q)\|_2^2
=
a+b+2|c|.
\]
Hence
\[
\inf_{0<s<1}
\lambda_{\max}\!\left(
\frac{pp^\top}{s}
+
\frac{qq^\top}{1-s}
\right)
\ge
a+b+2|c|.
\]

It remains to show the reverse inequality. If $p$ and $q$ are both nonzero, set
\[
s_*:=\frac{a+|c|}{a+b+2|c|}.
\]
Then
\[
1-s_*=\frac{b+|c|}{a+b+2|c|}.
\]
A direct computation gives
\[
\left(
\frac{pp^\top}{s_*}
+
\frac{qq^\top}{1-s_*}
\right)(p+\tau q)
=
(a+b+2|c|)(p+\tau q).
\]
Moreover, the matrix
\[
M:=\frac{pp^\top}{s_*}+\frac{qq^\top}{1-s_*}
\]
has rank at most two and is supported on $\operatorname{span}\{p,q\}$. On the orthogonal complement of this span, $M$ is zero. Hence $M$ has at most two nonzero eigenvalues $\lambda_1(M) \ge \lambda_2(M) \ge 0$, and
\[
\lambda_1(M)+\lambda_2(M)
=
\operatorname{tr}(M)
=
\frac{a}{s_*}+\frac{b}{1-s_*}.
\]
Substituting $s_*$ and $1-s_*$,
\[
\operatorname{tr}(M)
=
\frac{a(a+b+2|c|)}{a+|c|}+\frac{b(a+b+2|c|)}{b+|c|}
=
(a+b+2|c|)\left[\frac{a}{a+|c|}+\frac{b}{b+|c|}\right].
\]
Since $\frac{a}{a+|c|} \le 1$ and $\frac{b}{b+|c|} \le 1$, we have $\operatorname{tr}(M) \le 2(a+b+2|c|)$. Combined with the lower bound $\lambda_1(M) \ge a+b+2|c|$ from before and the eigenvector identity $M(p+\tau q) = (a+b+2|c|)(p+\tau q)$, we conclude:
\[
\lambda_2(M) = \operatorname{tr}(M) - \lambda_1(M) \le \operatorname{tr}(M) - (a+b+2|c|) \le a+b+2|c|.
\]
So $a+b+2|c|$ is sandwiched between $\lambda_1(M)$ and $\lambda_2(M)$ from above and below, and the eigenvector identity forces it to be one of the eigenvalues. Combined with $\lambda_1(M) \ge a+b+2|c| \ge \lambda_2(M)$, it must be $\lambda_1(M) = a+b+2|c|$, i.e.,
\[
\lambda_{\max}(M) = a+b+2|c|.
\]

If $p=0$ then $W(\alpha_1,\alpha_2)=\begin{bmatrix}0 & q/\alpha_2\end{bmatrix}$, so $\|W\|_{\op}=\|q\|_2/\alpha_2$. Then
\[
\max\{\|W\|_{\op},\|\alpha\|_2\}=\max\{\|q\|_2/\alpha_2,\sqrt{\alpha_1^2+\alpha_2^2}\},
\]
and taking $\alpha_1\to 0^+$ and optimizing $\alpha_2>0$ gives $R_{\op}(\mathcal{C}(0,q))=\|q\|_2^{1/2}=\bigl(\|p\|_2^2+\|q\|_2^2+2|p^\top q|\bigr)^{1/4}$. The case $q=0$ is symmetric. Taking $1/4$-roots gives the wanted identity.
\end{proof}

Now we do analytic computation to compute how each component is preferred by one optimizer or another. We let $y = \textbf{1}_{2d}$ which leads to easier calculation. 

\begin{proposition}
\label{prop4:formula_Rop_Rinf}
Let $A\in\mathbb R^{d\times d}$ be invertible and
$y=\mathbf 1_{2d}$. Then for every $\sigma\in\{\pm1\}^d$,
\[
R_\infty(\mathcal C_{\sigma\mid-\sigma})
=
\|A^{-1}\sigma\|_\infty^{1/2},
\]
and
\[
R_{\mathrm{op}}(\mathcal C_{\sigma\mid-\sigma})
=
\sqrt{2\|A^{-1}\sigma\|_2}.
\]
\end{proposition}

\begin{proof}
Since $y=\mathbf 1_{2d}$, we have
\[
y_\sigma=\sigma,
\qquad
y_{-\sigma}=-\sigma.
\]
Thus, on $\mathcal C_{\sigma\mid-\sigma}$,
\[
p_\sigma=A^{-1}y_\sigma=A^{-1}\sigma,
\qquad
q_\sigma=A^{-1}y_{-\sigma}=-A^{-1}\sigma.
\]
By the closed-form formulas for $\mathcal C(p,q)$,
\[
R_\infty(\mathcal C(p,q))
=
\max\{\|p\|_\infty^{1/2},\|q\|_\infty^{1/2}\},
\]
and
\[
R_{\mathrm{op}}(\mathcal C(p,q))
=
\max\{\|p+q\|_2,\|p-q\|_2\}^{1/2}.
\]
Substituting $p=p_\sigma=A^{-1}\sigma$ and $q=q_\sigma=-A^{-1}\sigma$ gives
\[
R_\infty(\mathcal C_{\sigma\mid-\sigma})
=
\|A^{-1}\sigma\|_\infty^{1/2},
\]
and
\[
R_{\mathrm{op}}(\mathcal C_{\sigma\mid-\sigma})
=
\|2A^{-1}\sigma\|_2^{1/2}
=
\sqrt{2\|A^{-1}\sigma\|_2}.
\]
\end{proof}

Given the exact formula of $R_{\infty}(\mathcal{C}_{\sigma | -\sigma})$ and $R_{\op}(\mathcal{C}_{\sigma | -\sigma})$, we can find which connected components will have the least $R_{\infty}$ and the least $R_{\op}$. These connected components are most preferred by Adam and Muon, respectively. 

\begin{proposition}
\label{prop5:computation}
Let $d\ge2$ and fix $1<L<\sqrt d$.
Denote
\[
h_1:=\mathbf 1=(1,\dots,1)^\top,
\qquad
h_2:=(1,-1,\dots,-1)^\top,
\]
and let
\[
B =
\begin{bmatrix}
\dfrac{1+L}{2}
&
\dfrac{1-L}{2(d-1)}
&
\dfrac{1-L}{2(d-1)}
&
\cdots
&
\dfrac{1-L}{2(d-1)}
\\[6pt]
\dfrac12
&
1-\dfrac{1}{2(d-1)}
&
-\dfrac{1}{2(d-1)}
&
\cdots
&
-\dfrac{1}{2(d-1)}
\\[6pt]
\dfrac12
&
-\dfrac{1}{2(d-1)}
&
1-\dfrac{1}{2(d-1)}
&
\cdots
&
-\dfrac{1}{2(d-1)}
\\
\vdots
&
\vdots
&
\vdots
&
\ddots
&
\vdots
\\[4pt]
\dfrac12
&
-\dfrac{1}{2(d-1)}
&
-\dfrac{1}{2(d-1)}
&
\cdots
&
1-\dfrac{1}{2(d-1)}
\end{bmatrix}.
\]
Set
\[
A:=B^{-1},
\qquad
y=\mathbf 1_{2d}.
\]
Then $R_\infty(\mathcal C_{\sigma\mid-\sigma})$
is minimized exactly at $\sigma=\pm h_1$, while
$
R_{\mathrm{op}}(\mathcal C_{\sigma\mid-\sigma})
$
is minimized exactly at $\sigma=\pm h_2$. Moreover, 
Assume $L=\frac{\sqrt d}{2}\gg 2$. Then for sufficiently large $d \geq 16$,
\[
\min_{\sigma\in\{\pm1\}^d}
R_\infty(\mathcal C_{\sigma\mid-\sigma})
=
R_\infty(\mathcal C_{h_1\mid-h_1})
=
1, \quad 
\min_{\substack{\sigma\in\{\pm1\}^d\\ \sigma\neq \pm h_1}}
R_\infty(\mathcal C_{\sigma\mid-\sigma})
=
\left(
1+\frac{\sqrt d/2-1}{d-1}
\right)^{1/2},
\]
\[
\min_{\sigma\in\{\pm1\}^d}
R_{\mathrm{op}}(\mathcal C_{\sigma\mid-\sigma})
=
R_{\mathrm{op}}(\mathcal C_{h_2\mid-h_2})
=
d^{1/4}/\sqrt{2},
\]
and
\[
\min_{\substack{\sigma\in\{\pm1\}^d\\ \sigma\neq \pm h_2}}
R_{\mathrm{op}}(\mathcal C_{\sigma\mid-\sigma})
=
\sqrt2
\left[
\left(
\frac{\sqrt d}{2}
-
\frac{\frac{\sqrt d}{2}-1}{d-1}
\right)^2
+
4-\frac{3}{d-1}
\right]^{1/4}.
\]
\end{proposition}

\begin{proof}
By \cref{prop4:formula_Rop_Rinf},
\[
R_\infty(\mathcal C_{\sigma\mid-\sigma})
=
\|B\sigma\|_\infty^{1/2},
\qquad
R_{\mathrm{op}}(\mathcal C_{\sigma\mid-\sigma})
=
\sqrt{2\|B\sigma\|_2}.
\]
Hence it suffices to minimize $\|B\sigma\|_\infty$ and $\|B\sigma\|_2$ over
$\sigma\in\{\pm1\}^d$.
Both norms are invariant under $\sigma\mapsto-\sigma$, since
$B(-\sigma)=-B\sigma$. Thus we may assume $\sigma_1=1$.

Let
\[
k:=|\{i\in\{2,\dots,d\}:\sigma_i=-1\}|,
\qquad
r:=\frac{k}{d-1}.
\]
Then $0\le r\le1$. From the definition of $B$,
\[
(B\sigma)_1=1+(L-1)r,
\]
and for $i\ge2$,
\[
(B\sigma)_i=
\begin{cases}
1+r, & \sigma_i=1,\\
-1+r, & \sigma_i=-1.
\end{cases}
\]

We first minimize $\|B\sigma\|_\infty$. If $r=0$, then $\sigma=h_1$ and
\[
B\sigma=\mathbf 1,
\]
so $\|B\sigma\|_\infty=1$. If $r>0$, then since $L>1$,
\[
|(B\sigma)_1|=1+(L-1)r>1.
\]
Thus the unique minimizer with $\sigma_1=1$ is $\sigma=h_1$. By sign symmetry,
the global minimizers are exactly $\sigma=\pm h_1$.

Now minimize $\|B\sigma\|_2$. The coordinate formula gives
\[
\|B\sigma\|_2^2
=
\bigl(1+(L-1)r\bigr)^2
+
(d-k-1)(1+r)^2
+
k(-1+r)^2.
\]
Using $k=(d-1)r$, the last two terms become
\[
(d-1)(1-r)(1+r)^2+(d-1)r(1-r)^2.
\]
Since
\[
(1-r)(1+r)^2+r(1-r)^2
=
1+2r-3r^2,
\]
we obtain
\[
\|B\sigma\|_2^2
=
\bigl(1+(L-1)r\bigr)^2
+
(d-1)(1+2r-3r^2).
\]
For $\sigma=h_2$, we have $r=1$, hence
\[
Bh_2=Le_1,
\qquad
\|Bh_2\|_2^2=L^2.
\]

For $r<1$, set $\delta:=1-r>0$. Then
\[
1+(L-1)r
=
L-(L-1)\delta,
\]
and
\[
1+2r-3r^2
=
4\delta-3\delta^2.
\]
Therefore
\[
\|B\sigma\|_2^2
=
\bigl(L-(L-1)\delta\bigr)^2
+
(d-1)(4\delta-3\delta^2).
\]
Subtracting $L^2$ gives
\[
\begin{aligned}
\|B\sigma\|_2^2-L^2
&=
-2L(L-1)\delta
+
(L-1)^2\delta^2
+
4(d-1)\delta
-
3(d-1)\delta^2
\\
&=
\delta
\left[
4(d-1)-2L(L-1)
+
\delta\bigl((L-1)^2-3(d-1)\bigr)
\right].
\end{aligned}
\]
We claim the bracket is strictly positive for every $\delta\in(0,1]$.

If
\[
(L-1)^2-3(d-1)\ge0,
\]
then the bracket is minimized at $\delta=0$, where it equals
\[
4(d-1)-2L(L-1).
\]
This is strictly positive because
\[
4(d-1)-2L(L-1)
>
4(d-1)-2d
=
2d-4
\ge0,
\]
and the inequality is strict since $L(L-1)<L^2<d$.

If
\[
(L-1)^2-3(d-1)<0,
\]
then the bracket is minimized at $\delta=1$, where it equals
\[
\begin{aligned}
&4(d-1)-2L(L-1)+(L-1)^2-3(d-1)
\\
&=
d-L^2
>
0.
\end{aligned}
\]
Therefore
\[
\|B\sigma\|_2^2>L^2
\]
whenever $r<1$. Thus the unique minimizer with $\sigma_1=1$ is $\sigma=h_2$. By sign symmetry,
the global minimizers are exactly $\sigma=\pm h_2$.

We now derive the runner-up formulas. For the $R_\infty$ runner-up, recall that
\[
\|B\sigma\|_\infty
=
\max\{1+(L-1)r,\ 1+r\},
\qquad
r=\frac{k}{d-1}.
\]
The minimum occurs at $r=0$. The smallest positive value of $r$ is
\[
r=\frac1{d-1}.
\]
Therefore
\[
\left(R_\infty^{(2)}\right)^2
=
\max\left\{
1+\frac{L-1}{d-1},
1+\frac1{d-1}
\right\}
=
1+\frac{\max\{L-1,1\}}{d-1}.
\]
Hence
\[
R_\infty^{(2)}
=
\left(
1+\frac{\max\{L-1,1\}}{d-1}
\right)^{1/2}.
\]

For the $R_{\mathrm{op}}$ runner-up, define
\[
F(r):=\|B\sigma\|_2^2
=
\bigl(1+(L-1)r\bigr)^2
+
(d-1)(1+2r-3r^2).
\]
Then
\[
R_{\mathrm{op}}(\mathcal C_{\sigma\mid-\sigma})
=
\sqrt2\,F(r)^{1/4}.
\]
The minimum occurs at $r=1$, where $F(1)=L^2$.
Thus the runner-up is the smallest value of $F(r)$ over
\[
r\in
\left\{
0,\frac1{d-1},\dots,1-\frac1{d-1}
\right\}.
\]
Equivalently, set $\delta:=1-r$. Then
\[
\delta\in
\left\{
\frac1{d-1},\frac2{d-1},\dots,1
\right\}.
\]
From the previous computation,
\[
F(r)-L^2
=
\delta
\left[
4(d-1)-2L(L-1)
+
\delta\bigl((L-1)^2-3(d-1)\bigr)
\right].
\]
As a function of $\delta$, this is a quadratic with leading coefficient
\[
(L-1)^2-3(d-1).
\]
Since $L^2<d$, we have
\[
(L-1)^2 < L^2 < d < 3(d-1)
\]
for $d\ge2$, and hence the quadratic is strictly concave. Therefore its minimum over the discrete interval occurs at one of the endpoints:
\[
\delta=\frac1{d-1}
\qquad\text{or}\qquad
\delta=1.
\]
Equivalently,
\[
r=1-\frac1{d-1}
\qquad\text{or}\qquad
r=0.
\]
Hence
\[
F^{(2)}
=
\min\left\{
F\left(1-\frac1{d-1}\right),
F(0)
\right\}.
\]
Now
\[
F(0)=d,
\]
and
\[
\begin{aligned}
F\left(1-\frac1{d-1}\right)
&=
\left(
1+(L-1)\left(1-\frac1{d-1}\right)
\right)^2
\\
&\quad+
(d-1)
\left[
1+2\left(1-\frac1{d-1}\right)
-
3\left(1-\frac1{d-1}\right)^2
\right]
\\
&=
\left(
L-\frac{L-1}{d-1}
\right)^2
+
4-\frac{3}{d-1}.
\end{aligned}
\]
Therefore
\[
F^{(2)}
=
\min\left\{
d,\,
\left(
L-\frac{L-1}{d-1}
\right)^2
+
4-\frac{3}{d-1}
\right\}.
\]
Since
\[
R_{\mathrm{op}}=\sqrt2\,F(r)^{1/4},
\]
we get
\[
R_{\mathrm{op}}^{(2)}
=
\sqrt2
\left[
\min\left\{
d,\,
\left(
L-\frac{L-1}{d-1}
\right)^2
+
4-\frac{3}{d-1}
\right\}
\right]^{1/4}.
\]

Finally, set $L=\sqrt d/2$. If $d\ge16$, then $L-1\ge1$, and hence
\[
R_\infty^{(2)}
=
\left(
1+\frac{\frac{\sqrt d}{2}-1}{d-1}
\right)^{1/2}.
\]
Moreover,
\[
R_{\mathrm{op}}(\mathcal C_{h_2\mid-h_2})
=
\sqrt{2L}
=
d^{1/4}.
\]
For sufficiently large $d$, one also has
\[
\left(
\frac{\sqrt d}{2}
-
\frac{\frac{\sqrt d}{2}-1}{d-1}
\right)^2
+
4-\frac{3}{d-1}
< d,
\]
so the minimum in the formula for $F^{(2)}$ is attained by the second term. Therefore
\[
R_{\mathrm{op}}^{(2)}
=
\sqrt2
\left[
\left(
\frac{\sqrt d}{2}
-
\frac{\frac{\sqrt d}{2}-1}{d-1}
\right)^2
+
4-\frac{3}{d-1}
\right]^{1/4}.
\]
This proves the claimed formulas.
\end{proof}

\begin{theorem}
\label{appxt1:Finitewidth_construction}
(\cref{t1:Finitewidth_construction} of the paper)
Assume $A$, $X$, $y$ are given as in \cref{prop5:computation}, and define
\[
R_{\op}^{(1)}
:=
\min_{\sigma\in\{\pm1\}^d}
R_{\op}(\mathcal C_{\sigma\mid-\sigma}),
\qquad
R_{\op}^{(2)}
:=
\min_{\substack{\sigma\in\{\pm1\}^d\\
R_{\op}(\mathcal C_{\sigma\mid-\sigma})>R_{\op}^{(1)}}}
R_{\op}(\mathcal C_{\sigma\mid-\sigma}),
\]
and similarly
\[
R_{\max}^{(1)}
:=
\min_{\sigma\in\{\pm1\}^d}
R_{\max}(\mathcal C_{\sigma\mid-\sigma}),
\qquad
R_{\max}^{(2)}
:=
\min_{\substack{\sigma\in\{\pm1\}^d\\
R_{\max}(\mathcal C_{\sigma\mid-\sigma})>R_{\max}^{(1)}}}
R_{\max}(\mathcal C_{\sigma\mid-\sigma}).
\]
Choose $\lambda_{\mathrm{Muon}},\lambda_{\mathrm{AdamW}}>0$ such that
\[
R_{\op}^{(1)}
\leq
\frac{1}{\lambda_{\mathrm{Muon}}}
\leq
R_{\op}^{(2)},
\quad
R_{\max}^{(1)}
\leq
\frac{1}{\lambda_{\mathrm{AdamW}}} \leq
R_{\max}^{(2)}.
\]
Then
\[
\mathcal{O}_{\|\cdot\|_{\op}}(\lambda_{\mathrm{Muon}})
\subseteq
\mathcal C_{h_2\mid-h_2}\sqcup \mathcal C_{-h_2\mid h_2},
\qquad
\mathcal{O}_{\|\cdot\|_{\max}}(\lambda_{\mathrm{AdamW}})
\subseteq
\mathcal C_{h_1\mid-h_1}\sqcup \mathcal C_{-h_1\mid h_1},
\]
and both sets are nonempty. Consequently, any two Muon-found solutions are connected by a zero-loss path in $\mathcal{O}$ up to permutation of the two neurons, and similarly for AdamW; but no Muon-found solution can be connected to an AdamW-found solution by any zero-loss path, even up to permutation. Moreover, for the squared loss $\mathcal{L}(x,y) = \tfrac{1}{2}\|x-y\|_2^2$, any continuous path connecting a point in $\mathcal{O}_{\|\cdot\|_{\op}}(\lambda_{\mathrm{Muon}})$ to a point in $\mathcal{O}_{\|\cdot\|_{\max}}(\lambda_{\mathrm{AdamW}})$ must pass through a parameter $(W,\alpha)$ with $\mathcal{L}((XW)_+\alpha, y) \geq \tfrac{1}{2}$.
\end{theorem}
 
\begin{proof}
From the definitions of $R_{\op}^{(1)}$ and $R_{\op}^{(2)}$, the condition
\[
R_{\op}^{(1)} \le \frac1{\lambda_{\mathrm{Muon}}} < R_{\op}^{(2)}
\]
implies that
\[
\mathcal{O}_{\|\cdot\|_{\op}}(\lambda_{\mathrm{Muon}})
\subseteq
\mathcal C_{h_2\mid-h_2}\sqcup \mathcal C_{-h_2\mid h_2}.
\]
Similarly,
\[
R_{\max}^{(1)} \le \frac1{\lambda_{\mathrm{AdamW}}} < R_{\max}^{(2)}
\]
implies
\[
\mathcal{O}_{\|\cdot\|_{\max}}(\lambda_{\mathrm{AdamW}})
\subseteq
\mathcal C_{h_1\mid-h_1}\sqcup \mathcal C_{-h_1\mid h_1}.
\]
Nonemptiness follows since the minima defining $R_{\op}^{(1)}$ and $R_{\max}^{(1)}$ are attained.
Permuting the two neurons maps $\mathcal C_{\sigma\mid-\sigma}$ to $\mathcal C_{-\sigma\mid\sigma}$. Since each $\mathcal C_{\sigma\mid-\sigma}$ is connected, any two points in $\mathcal{O}_{\|\cdot\|_{\op}}(\lambda_{\mathrm{Muon}})$ (resp. $\mathcal{O}_{\|\cdot\|_{\max}}(\lambda_{\mathrm{AdamW}})$) can be mapped into the same connected component, giving the same-optimizer connectivity claims.
Finally, $h_1\neq\pm h_2$ for $d\ge 2$, so the corresponding components remain distinct even modulo permutation. Therefore no permutation can place a point from $\mathcal{O}_{\|\cdot\|_{\max}}(\lambda_{\mathrm{AdamW}})$ in the same component as a point from $\mathcal{O}_{\|\cdot\|_{\op}}(\lambda_{\mathrm{Muon}})$, giving the cross-optimizer separation.
 
It remains to establish the loss-barrier claim. Let $\gamma:[0,1]\to \mathbb{R}^{d\times 2}\times \mathbb{R}^2$, $\gamma(\tau) = (W(\tau),\alpha(\tau))$, be a continuous path with $\gamma(0)\in \mathcal{O}_{\|\cdot\|_{\op}}(\lambda_{\mathrm{Muon}})$ and $\gamma(1)\in \mathcal{O}_{\|\cdot\|_{\max}}(\lambda_{\mathrm{AdamW}})$. By the cross-optimizer separation, $\gamma(0)$ lies in $\mathcal{C}_{\pm h_2\mid \mp h_2}$ and $\gamma(1)$ lies in $\mathcal{C}_{\pm h_1\mid \mp h_1}$, with sign vectors satisfying $\sigma(0) = \pm h_2 \neq \pm h_1 = \sigma(1)$, where $\sigma(\tau) := \mathrm{sign}(AW_{\cdot 1}(\tau))$ when defined. Hence there exists $\tau^*\in(0,1)$ and a coordinate $i\in[d]$ such that $A_{i\cdot}W_{\cdot 1}(\tau^*) = 0$.
 
At $\tau^*$, both rows $i$ and $i+d$ of $(XW_{\cdot 1}(\tau^*))_+$ vanish, since they equal $(A_{i\cdot}W_{\cdot 1}(\tau^*))_+$ and $(-A_{i\cdot}W_{\cdot 1}(\tau^*))_+$, respectively. Thus at $\tau^{*}$ at least one element of $(XW(t))_{+}\alpha(t)$ is 0, meaning the loss at $\tau^{*}$ is at least 1/2. This proves the barrier claim.
\end{proof}

\section{Experimental Details}\label{exp_appendix}

\subsection{Same-Optimizer Connectivity Experimental Details (\cref{sec:setup})}\label{observation_appendix}

\paragraph{Model configurations.} We base our experiments on the codebase of~\citet{theus2025generalizedlinearmodeconnectivity} and use HuggingFace \texttt{Trainer}'s default GPT-2 model architecture with 6 layers, 4 attention heads, embedding dimension 256, and inner dimension 1024 (30.5M parameters). We use the GPT-2 tokenizer (vocabulary size 50257) and gelu\_new activation.

\paragraph{Optimizer implementation.} For optimizer implementation, we use torch's default AdamW optimizer from \texttt{torch.optim.AdamW}. For Muon, we follow Keller's implementation from \url{https://github.com/KellerJordan/Muon} and re-implement ourself. Since GPT-2 model in HuggingFace \texttt{Trainer} class originally concatenates attention $Q,K,V$ as a single matrix parameter, we manually split it up into three blocks and apply Muon independently on each. We further write custom code to split matrix and non-matrix parameters, and always apply AdamW with learning rate $6e-4$, $(\beta_1,\beta_2)=(0.85,0.999)$, weight decay $0.1$ to non-matrix parameters. We follow tradition in \citep{jordan2024muon} and pair Muon with WSD lr scheduler where we set warmup ratio to be $0.0$. Note this is also consistent with \citep{wen2025fantasticpretrainingoptimizers}. For AdamW, we use cosine lr scheduler and set warmup ratio to be $0.05$, which is default in \citep{theus2025generalizedlinearmodeconnectivity}. For hyperparameters for matrix optimizers, we tune learning rate and weight decay for Muon and AdamW independently for each dataset, and take the optimal hyperparameter set with lowest test loss.
\paragraph{Dataset details.} Our experiments cover five datasets: \emph{enwik8} dataset \citep{enwik8} contains in total $45$M training tokens; \emph{WikiText-103} dataset \citep{merity2016pointersentinelmixturemodels} contains in total $100$M tokens; \emph{Stories} dataset \citep{trinh2019simplemethodcommonsensereasoning} contains in total $200$M tokens; \emph{BookCorpus} dataset \citep{zhu2015aligningbooksmoviesstorylike} has around $985$M tokens in total; \emph{One Billion Word} dataset \citep{chelba2014billionwordbenchmarkmeasuring} has around $1.1$B tokens in total.
\paragraph{Algorithm.} We present the full algorithm from canonicalization to alignment in Algorithm \ref{condensed_alg}.
\begin{algorithm}[ht!]
\caption{Symmetric alignment with \colorbox{pink!30}{polychain} interpolation}
\label{condensed_alg}
\begin{algorithmic}[1]
\REQUIRE Canonicalized models $\bm{\theta}_A, \bm{\theta}_B$; Dataset $\mathcal{D}$; Iterations $N_{\mathrm{iter}}$; Adam optimizer ($\texttt{lr}=\eta$).
\STATE Apply weight matching to coarsely align $\bm{\theta}_B$. Initialize $\tilde{\bm{P}}_{\mathrm{mlp}}, \tilde{\bm{P}}_{\mathrm{head}}, \tilde{\bm{O}}$ near identity.
\STATE \hlpink{Initialize bend-point $\bm{\theta}_C \leftarrow \tfrac{1}{2}(\bm{\theta}_A + \bm{\theta}_B)$.}
\FOR{$t = 1$ \textbf{to} $N_{\mathrm{iter}}$}
    \STATE Project: $\{\bm{P}_{\mathrm{mlp}}, \bm{P}_{\mathrm{head}}, \bm{O}\} \leftarrow \{\textsc{ProjPerm}(\tilde{\bm{P}}_{\mathrm{mlp}}),\, \textsc{ProjPerm}(\tilde{\bm{P}}_{\mathrm{head}}),\, \textsc{ProjOrth}(\tilde{\bm{O}})\}$.
    \STATE Align: $\bm{\theta}_B^{\mathrm{aligned}} \leftarrow \pi(\bm{\theta}_B;\, \bm{P}_{\mathrm{mlp}}, \bm{P}_{\mathrm{head}}, \bm{O})$. \hfill $\triangleright$ $\pi$ applies transformations
    \STATE Sample $\lambda \sim \text{Uniform}(0.4, 0.6)$.
    \STATE Sample batch $B = \{(\bm{X}_i, \bm{Y}_i)\}_{i=1}^{|B|}$ from $\mathcal{D}$.
    \STATE \hlpink{Interpolate via polychain:}
\[\hlpink{$\displaystyle
    \bm{\theta}_{\mathrm{interp}} = \begin{cases} (1-2\lambda)\,\bm{\theta}_A + 2\lambda\,\bm{\theta}_C & \lambda \in [0, 0.5] \\ (2-2\lambda)\,\bm{\theta}_C + (2\lambda-1)\,\bm{\theta}_B^{\mathrm{aligned}} & \lambda \in (0.5, 1] \end{cases}
$}\]
    \STATE Objective: $\mathcal{J} \leftarrow \frac{1}{|B|}\sum_{(\bm{X},\bm{Y})\in B} \mathcal{L}_{\mathrm{CE}}(\bm{\theta}_{\mathrm{interp}};\, \bm{X}, \bm{Y})$.
    \STATE Gradients: $(\bm{g}_{\tilde{\bm{P}}_\mathrm{mlp}}, \bm{g}_{\tilde{\bm{P}}_\mathrm{head}}, \bm{g}_{\tilde{\bm{O}}}, \bm{g}_{\bm{\theta}_C}) \leftarrow \nabla_{(\tilde{\bm{P}}_\mathrm{mlp}, \tilde{\bm{P}}_\mathrm{head}, \tilde{\bm{O}}, \bm{\theta}_C)}\,\mathcal{J}$. \hfill $\triangleright$ STE for $\tilde{\bm{P}}_{\mathrm{mlp}}, \tilde{\bm{P}}_{\mathrm{head}}$
    \STATE Update: For $k \in \{\bm{P}_\mathrm{mlp}, \bm{P}_\mathrm{head}, \bm{O}\}$, $\tilde k \leftarrow \mathrm{Adam}(\tilde k, \bm{g}_{\tilde k}, \eta)$;  \hlpink{$\bm{\theta}_C \leftarrow \mathrm{Adam}(\bm{\theta}_C, \bm{g}_{\bm{\theta}_C}, \eta)$}.
\ENDFOR
\STATE Final projections: $\bm{P}_{\mathrm{mlp}}^* \leftarrow \textsc{ProjPerm}(\tilde{\bm{P}}_\mathrm{mlp})$, $\bm{P}_{\mathrm{head}}^* \leftarrow \textsc{ProjPerm}(\tilde{\bm{P}}_\mathrm{head})$, $\bm{O}^* \leftarrow \textsc{ProjOrth}(\tilde{\bm{O}})$.
\RETURN $\bm{P}_{\mathrm{mlp}}^*, \bm{P}_{\mathrm{head}}^*, \bm{O}^*$, \hlpink{$\bm{\theta}_C$}.
\end{algorithmic}
\end{algorithm}
\paragraph{Hyperparameter details.} For model training, we first train AdamW on One Billion Word (largest dataset) with early stop enabled, and the training always exits after around $13000$ steps. We thus impose a uniform $15000$ training steps for all model training experiments, and we find $15000$ steps are also compliant with Chinchilla's rule. This amounts to $21\sim 22$ training epochs for enwik8; $\sim 10$ training epochs for WikiText-103; $\sim 5$ training epochs for Stories; around $1.1$ epoch over all training tokens for BookCorpus; around one single pass over the One Billion Word dataset.

We record learning rates sweep results for different optimizers for different datasets in \cref{lr_search_adamw_enwik8,lr_search_muon_enwik8,lr_search_adamw_wiki,lr_search_muon_wiki,lr_search_adamw_stories,lr_search_muon_stories,lr_search_adamw_bookcorpus,lr_search_muon_bookcorpus,lr_search_adamw_lm1b,lr_search_muon_lm1b} and weight decay sweep results in \cref{wd_search_adamw_enwik8,wd_search_muon_enwik8,wd_search_adamw_wiki,wd_search_muon_wiki,wd_search_adamw_stories,wd_search_muon_stories,wd_search_adamw_bookcorpus,wd_search_muon_bookcorpus,wd_search_adamw_lm1b,wd_search_muon_lm1b}.

For alignment training, despite we add a polychain component and thus introduce a new copy of trainable parameter, we stick to default alignment training setting choice in \citep{theus2025generalizedlinearmodeconnectivity}. We train for $10$ epochs for enwik8; $10$ epochs for WkiText-103; $10$ epochs for Stories; $1$ epoch for BookCorpus; $1$ epoch for One Billion Word.
\begin{table}[ht!]
  \centering
  \caption{learning rate sweep (AdamW, enwik8)}\label{lr_search_adamw_enwik8}
  \small
  \begin{tabular}{cccccc}
    \toprule
    \multicolumn{6}{c}{\text{lr: validation loss}}\\ 
    \midrule
    {1e-4: $3.259234$} & {1e-3: $3.040800$} & {5e-3: $2.987674$} & {6e-3: $2.949275$} & {7e-3: $2.974285$} & {8e-3: $2.971421$} \\
    {9e-3: $2.946757$} & {\roundedcell{1e-2: $2.940545$}} & {2e-2: $3.879081$} & {3e-2: $4.497614$} & {4e-2: $5.288924$} & {5e-2: $5.521785$} \\
    {1e-1: $5.795908$}  \\
    \bottomrule
  \end{tabular}
\end{table}

\begin{table}[ht!]
  \centering
  \caption{weight decay sweep (AdamW, enwik8)}\label{wd_search_adamw_enwik8}
  \small
  \begin{tabular}{cccccc}
    \toprule
    \multicolumn{6}{c}{\text{wd: validation loss}}\\ 
    \midrule
    {1e-5: $2.982424$} & {1e-4: $2.984059$} & {1e-3: $2.981161$} & {1e-2: $ 2.974202$} & {\roundedcell{1e-1: $2.940545$}} & {1e1: $3.063595$} \\
    \bottomrule
  \end{tabular}
\end{table}
\begin{table}[ht!]
  \centering
  \caption{learning rate sweep (Muon, enwik8)}\label{lr_search_muon_enwik8}
  \small
  \begin{tabular}{cccccc}
    \toprule
    \multicolumn{6}{c}{\text{lr: validation loss}}\\ 
    \midrule
    {1e-4: $3.459713$} & {1e-3: $3.138396$} & {5e-3: $3.021547$} & {1e-2: $2.991482$} & {2e-2: $2.977345$} & {\roundedcell{3e-2: $2.955646$}} \\
    {4e-2: $2.965786$} & {5e-2: $2.966031$} & {6e-2: $2.974766$} & {7e-2: $2.981924$} & {8e-2: $2.993871$} & {9e-2: $2.995507$} \\
    {1e-1: $3.002056$}  \\
    \bottomrule
  \end{tabular}
\end{table}

\begin{table}[ht!]
  \centering
  \caption{weight decay sweep (Muon, enwik8)}\label{wd_search_muon_enwik8}
  \small
  \begin{tabular}{cccccc}
    \toprule
    \multicolumn{6}{c}{\text{wd: validation loss}}\\ 
    \midrule
    {1e-5: $2.995775$} & {1e-4: $3.001897$} & {1e-3: $2.995257$} & {1e-2: $2.976258$} & {\roundedcell{1e-1: $2.955646$}} & {1e1: $3.194322$} \\
    \bottomrule
  \end{tabular}
\end{table}
\begin{table}[ht!]
  \centering
  \caption{learning rate sweep (AdamW, WikiText-103)}\label{lr_search_adamw_wiki}
  \small
  \begin{tabular}{cccccc}
    \toprule
    \multicolumn{6}{c}{\text{lr: validation loss}}\\ 
    \midrule
    {1e-4: $3.716662$} & {5e-4: $3.548120$} & {1e-3: $3.504182$} & {2e-3: $3.487067$} & {\roundedcell{3e-3: $3.484125$}} & {4e-3: $3.485193$} \\
    {5e-3: $3.486527$} & {6e-3: $3.487567$} & {7e-3: $3.487229$} & {8e-3: $3.512726$} & {9e-3: $3.865148$} & {1e-2: $4.216534$} \\
    {1e-1: $6.215991$}  \\
    \bottomrule
  \end{tabular}
\end{table}

\begin{table}[ht!]
  \centering
  \caption{weight decay sweep (AdamW, WikiText-103)}\label{wd_search_adamw_wiki}
  \small
  \begin{tabular}{cccccc}
    \toprule
    \multicolumn{6}{c}{\text{wd: validation loss}}\\ 
    \midrule
    {1e-5: $3.504556$} & {1e-4: $3.504628$} & {1e-3: $3.501938$} & {1e-2: $3.502638$} & {\roundedcell{1e-1: $3.484125$}} & {1e1: $3.548662$} \\
    \bottomrule
  \end{tabular}
\end{table}
\begin{table}[ht!]
  \centering
  \caption{learning rate sweep (Muon, WikiText-103)}\label{lr_search_muon_wiki}
  \small
  \begin{tabular}{cccccc}
    \toprule
    \multicolumn{6}{c}{\text{lr: validation loss}}\\ 
    \midrule
    {1e-4: $3.762618$} & {1e-3: $3.502252$} & {2e-3: $3.476695$} & {3e-3: $3.463665$} & {4e-3: $3.456732$} & {\roundedcell{5e-3: $3.456051$}} \\
    {6e-3: $3.459424$} & {7e-3: $3.462454$} & {8e-3: $3.464876$} & {9e-3: $3.464611$} & {1e-2: $3.472285$} & {5e-2: $3.563784$} \\
    {1e-1: $3.636153$}  \\
    \bottomrule
  \end{tabular}
\end{table}

\begin{table}[ht!]
  \centering
  \caption{weight decay sweep (Muon, WikiText-103)}\label{wd_search_muon_wiki}
  \small
  \begin{tabular}{cccccc}
    \toprule
    \multicolumn{6}{c}{\text{wd: validation loss}}\\ 
    \midrule
    {1e-5: $3.466826$} & {1e-4: $3.465939$} & {1e-3: $3.465568$} & {1e-2: $3.462012$} & {\roundedcell{1e-1: $3.456051$}} & {1e1: $3.694885$} \\
    \bottomrule
  \end{tabular}
\end{table}
\begin{table}[ht!]
  \centering
  \caption{learning rate sweep (AdamW, Stories)}\label{lr_search_adamw_stories}
  \small
  \begin{tabular}{cccccc}
    \toprule
    \multicolumn{6}{c}{\text{lr: validation loss}}\\ 
    \midrule
    {1e-4: $4.250962$} & {5e-4: $4.114538$} & {1e-3: $4.091298$} & {2e-3: $4.078916$} & {3e-3: $4.074738$} & {\roundedcell{4e-3: $4.074564$}} \\
    {5e-3: $4.075013$} & {6e-3: $4.101315$} & {7e-3: $4.100038$} & {8e-3: $4.540163$} & {9e-3: $4.577148$} & {1e-2: $4.745496$} \\
    {1e-1: $5.940750$}  \\
    \bottomrule
  \end{tabular}
\end{table}

\begin{table}[ht!]
  \centering
  \caption{weight decay sweep (AdamW, Stories)}\label{wd_search_adamw_stories}
  \small
  \begin{tabular}{cccccc}
    \toprule
    \multicolumn{6}{c}{\text{wd: validation loss}}\\ 
    \midrule
    {1e-5: $4.100180$} & {1e-4: $4.096759$} & {1e-3: $4.099617$} & {1e-2: $4.093496$} & {\roundedcell{1e-1: $4.074564$}} & {1e1: $4.140388$} \\
    \bottomrule
  \end{tabular}
\end{table}
\begin{table}[ht!]
  \centering
  \caption{learning rate sweep (Muon, Stories)}\label{lr_search_muon_stories}
  \small
  \begin{tabular}{cccccc}
    \toprule
    \multicolumn{6}{c}{\text{lr: validation loss}}\\ 
    \midrule
    {1e-4: $4.283535$} & {1e-3: $4.101864$} & {2e-3: $4.064723$} & {3e-3: $4.058753$} & {\roundedcell{4e-3: $4.057592$}} & {5e-3: $4.058157$} \\
    {6e-3: $4.060997$} & {7e-3: $4.062566$} & {8e-3: $4.064774$} & {9e-3: $4.067385$} & {1e-2: $4.072307$} & {5e-2: $4.153174$} \\
    {1e-1: $4.210531$}  \\
    \bottomrule
  \end{tabular}
\end{table}

\begin{table}[ht!]
  \centering
  \caption{weight decay sweep (Muon, Stories)}\label{wd_search_muon_stories}
  \small
  \begin{tabular}{cccccc}
    \toprule
    \multicolumn{6}{c}{\text{wd: validation loss}}\\ 
    \midrule
    {1e-5: $4.065752$} & {1e-4: $4.065660$} & {1e-3: $4.065554$} & {1e-2: $4.062541$} & {\roundedcell{1e-1: $4.057592$}} & {1e1: $4.251467$} \\
    \bottomrule
  \end{tabular}
\end{table}
\begin{table}[ht!]
  \centering
  \caption{learning rate sweep (AdamW, BookCorpus)}\label{lr_search_adamw_bookcorpus}
  \small
  \begin{tabular}{cccccc}
    \toprule
    \multicolumn{6}{c}{\text{lr: validation loss}}\\ 
    \midrule
    {1e-4: $3.498248$} & {5e-4: $3.362210$} & {1e-3: $3.335717$} & {2e-3: $3.319726$} & {3e-3: $3.317407$} & {\roundedcell{4e-3: $3.314930$}} \\
    {5e-3: $3.327077$} & {6e-3: $3.330844$} & {7e-3: $3.334404$} & {8e-3: $3.336071$} & {9e-3: $3.349094$} & {1e-2: $3.852602$} \\
    {1e-1: $5.223763$}  \\
    \bottomrule
  \end{tabular}
\end{table}

\begin{table}[ht!]
  \centering
  \caption{weight decay sweep (AdamW, BookCorpus)}\label{wd_search_adamw_bookcorpus}
  \small
  \begin{tabular}{cccccc}
    \toprule
    \multicolumn{6}{c}{\text{wd: validation loss}}\\ 
    \midrule
    {1e-5: $3.337265$} & {1e-4: $3.337102$} & {1e-3: $3.334261$} & {1e-2: $3.332119$} & {\roundedcell{1e-1: $3.314930$}} & {1e1: $3.388930$} \\
    \bottomrule
  \end{tabular}
\end{table}
\begin{table}[ht!]
  \centering
  \caption{learning rate sweep (Muon, BookCorpus)}\label{lr_search_muon_bookcorpus}
  \small
  \begin{tabular}{cccccc}
    \toprule
    \multicolumn{6}{c}{\text{lr: validation loss}}\\ 
    \midrule
    {1e-4: $3.526606$} & {1e-3: $3.347175$} & {2e-3: $3.315645$} & {3e-3: $3.308063$} & {4e-3: $3.307607$} & {\roundedcell{5e-3: $3.306859$}} \\
    {6e-3: $3.308815$} & {7e-3: $3.311170$} & {8e-3: $3.314373$} & {9e-3: $3.315783$} & {1e-2: $3.318446$} & {5e-2: $3.394219$} \\
    {1e-1: $3.448584$}  \\
    \bottomrule
  \end{tabular}
\end{table}

\begin{table}[ht!]
  \centering
  \caption{weight decay sweep (Muon, BookCorpus)}\label{wd_search_muon_bookcorpus}
  \small
  \begin{tabular}{cccccc}
    \toprule
    \multicolumn{6}{c}{\text{wd: validation loss}}\\ 
    \midrule
    {1e-5: $3.310488$} & {1e-4: $3.310778$} & {1e-3: $3.310378$} & {1e-2: $3.307313$} & {\roundedcell{1e-1: $3.306859$}} & {1e1: $3.496961$} \\
    \bottomrule
  \end{tabular}
\end{table}
\begin{table}[ht!]
  \centering
  \caption{learning rate sweep (AdamW, One Billion Word)}\label{lr_search_adamw_lm1b}
  \small
  \begin{tabular}{cccccc}
    \toprule
    \multicolumn{6}{c}{\text{lr: validation loss}}\\ 
    \midrule
    {1e-4: $4.179953$} & {5e-4: $4.036719$} & {6e-4: $4.029493$} & {7e-4: $4.023592$} & {8e-4: $4.019736$} & {9e-4: $4.016547$} \\
    {1e-3: $4.013754$} & {2e-3: $4.003098$} & {3e-3: $4.001358$} & {\roundedcell{4e-3: $3.998650$}} & {5e-3: $4.030982$} & {1e-2: $4.609152$} \\
    {1e-1: $6.124138$}  \\
    \bottomrule
  \end{tabular}
\end{table}

\begin{table}[ht!]
  \centering
  \caption{weight decay sweep (AdamW, One Billion Word)}\label{wd_search_adamw_lm1b}
  \small
  \begin{tabular}{cccccc}
    \toprule
    \multicolumn{6}{c}{\text{wd: validation loss}}\\ 
    \midrule
    {1e-5: $4.023601$} & {1e-4: $4.056160$} & {1e-3: $4.024538$} & {1e-2: $4.019354$} & {\roundedcell{1e-1: $3.998650$}} & {1e1: $4.080319$} \\
    \bottomrule
  \end{tabular}
\end{table}
\begin{table}[ht!]
  \centering
  \caption{learning rate sweep (Muon, One Billion Word)}\label{lr_search_muon_lm1b}
  \small
  \begin{tabular}{cccccc}
    \toprule
    \multicolumn{6}{c}{\text{lr: validation loss}}\\ 
    \midrule
    {1e-4: $4.210103$} & {1e-3: $4.017325$} & {2e-3: $3.986968$} & {\roundedcell{3e-3: $3.980139$}} & {4e-3: $3.981803$} & {5e-3: $3.984024$} \\
    {6e-3: $3.986329$} & {7e-3: $3.990405$} & {8e-3: $3.990673$} & {9e-3: $3.994923$} & {1e-2: $3.997182$} & {5e-2: $4.095955$} \\
    {1e-1: $4.196278$}  \\
    \bottomrule
  \end{tabular}
\end{table}

\begin{table}[ht!]
  \centering
  \caption{weight decay sweep (Muon, One Billion Word)}\label{wd_search_muon_lm1b}
  \small
  \begin{tabular}{cccccc}
    \toprule
    \multicolumn{6}{c}{\text{wd: validation loss}}\\ 
    \midrule
    {1e-5: $3.991678$} & {1e-4: $3.991680$} & {1e-3: $3.991378$} & {1e-2: $3.988833$} & {\roundedcell{1e-1: $3.980139$}} & {1e1: $4.162952$} \\
    \bottomrule
  \end{tabular}
\end{table} 
\paragraph{Compute.} We run all experiments on NVIDIA H100 GPUs.
\subsection{Spectral Phase Transition Details (\cref{sec::inter_exp})}\label{appendix::inter_exp}

For the spectral phase transition experiment (\cref{fig:spectral_same_optimizer} and \cref{fig:spectral_adamw_muon}), we use a larger model (8 layers, embedding dim 1024, inner dim 4096) trained on One Billion Word with optimally tuned hyperparameters (AdamW: lr=$1\mathrm{e}{-3}$, wd=$0.1$; Muon: lr=$5\mathrm{e}{-3}$, wd=$0.01$). These hyperparameters are optimally tuned individually using grid search, see \cref{muon_sweep_large,adamw_sweep_large} for hyperparameter sweeping results. After symmetric alignment, we plot singular value histograms at interpolation coefficients $t \in \{0, 0.1, 0.3, 0.5, 0.7, 0.9, 1.0\}$. Note that both endpoint models are cast to canonical form (LayerNorm coefficients absorbed into weights) before alignment, which affects the SVD distribution but not the model's mathematical function. In figures presented in our main paper, we always train the midpoint $\theta_C$ with Adam using default hyperparameters in \citep{theus2025generalizedlinearmodeconnectivity}. We also present in \cref{adamw-muon-muon} where we train midpoint with Muon, and we observe similar spectral phase transition there.
\begin{figure}[!htbp]
    \centering
    \includegraphics[width=0.95\linewidth]{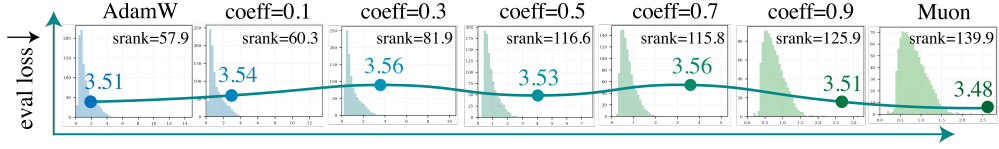}
    \caption{
    Spectrum along the AdamW$\rightarrow$Muon connectivity path.
    Each panel shows the singular value histogram at a different interpolation coefficient $t$. $\theta_C$ trained with Muon.
    }\label{adamw-muon-muon}
\end{figure}
\begin{table}[ht!]
  \centering
  \caption{learning rate and weight decay sweep (Muon, One Billion Word, large)}\label{muon_sweep_large}
  \small
  \begin{tabular}{l*{7}{c}}
    \toprule
    & \multicolumn{7}{c}{lr} \\
    \cmidrule(lr){2-8}
    wd & $1$ & $10^{-1}$ & $10^{-2}$ & $5\!\times\!10^{-3}$ & $10^{-3}$ & $10^{-4}$ & $10^{-5}$ \\
    \midrule
    $0.0$    & $6.78$ & $5.37$ & $3.49$ & $3.49$ & $3.55$ & $3.85$ & $4.27$ \\
    $0.01$ & $6.71$ & $4.84$ & $3.49$ & \roundedcell{$3.48$} & $3.54$ & $3.85$ & $4.27$ \\
    $0.1$  & $4.18$ & $3.82$ & $3.55$ & $3.51$ & $3.52$ & $3.84$ & $4.27$ \\
    $1.0$    & $4.89$ & $4.35$ & $3.95$ & $3.84$ & $3.69$ & $3.80$ & $4.27$ \\
    \bottomrule
  \end{tabular}
\end{table}

\begin{table}[ht!]
  \centering
  \caption{learning rate and weight decay sweep (AdamW, One Billion Word, large)}\label{adamw_sweep_large}
  \small
  \begin{tabular}{l*{6}{c}}
    \toprule
    & \multicolumn{6}{c}{lr} \\
    \cmidrule(lr){2-7}
    wd & $10^{-2}$ & $5\!\times\!10^{-3}$ & $10^{-3}$ & $5\!\times\!10^{-4}$ & $10^{-4}$ & $10^{-5}$ \\
    \midrule
    $0.0$    & $5.59$ & $5.62$ & $3.526$ & $3.530$ & $3.63$ & $4.00$ \\
    $0.01$ & $5.63$ & $5.61$ & $3.52$  & $3.53$  & $3.63$ & $4.00$ \\
    $0.1$  & $5.68$ & $5.66$ & \roundedcell{$3.51$}  & $3.52$  & $3.63$ & $4.00$ \\
    $1.0$    & $4.49$ & $4.85$ & $3.52$  & $3.52$  & $3.61$ & $4.00$ \\
    \bottomrule
  \end{tabular}
\end{table}
\subsection{Additional Out-Of-Distribution Results (\cref{sec:ood})}\label{sec::ood_appendix}
\begin{figure}[ht!]
    \centering
    \includegraphics[width=0.85\textwidth]{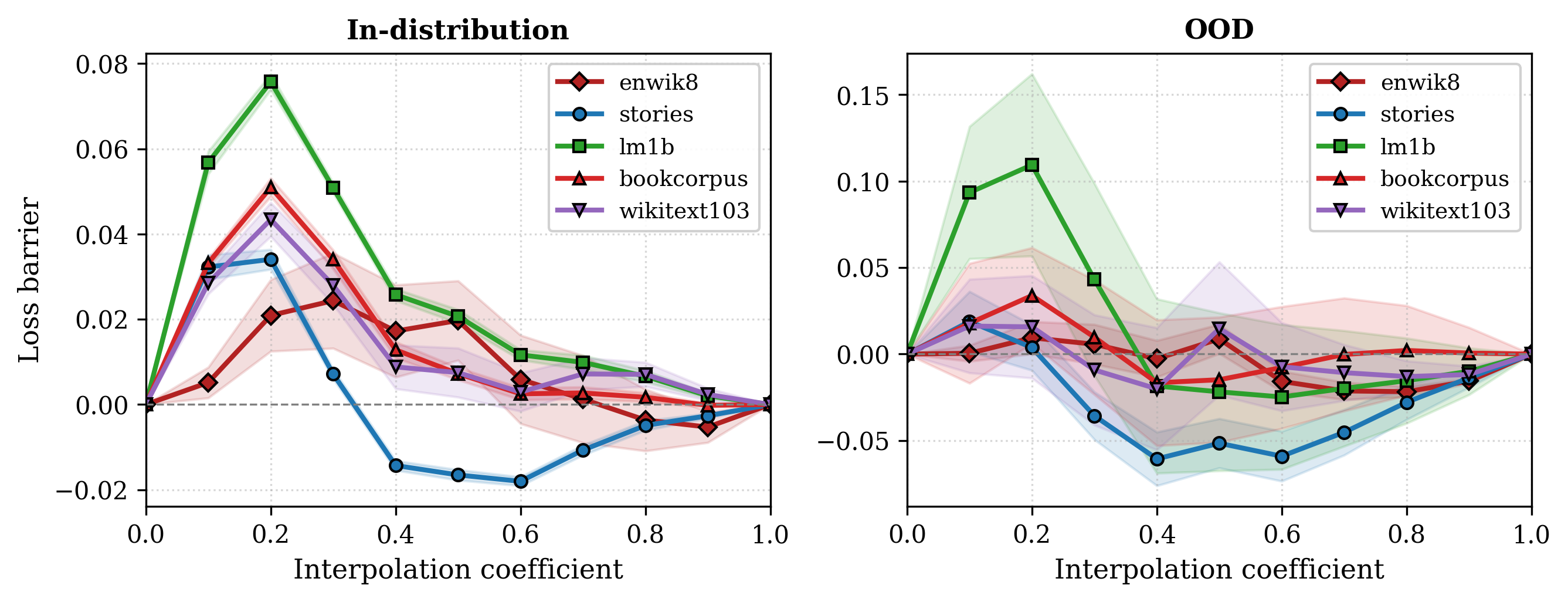}
    \caption{Out-of-distribution experiments for additional datasets. }\label{ood_additional}
\end{figure}
We present out-of-distribution probe for more datasets in Figure \ref{ood_additional}. For enwik8, we consider in-distribution dataset to be enwik8 and out-of-distribution dataset to be Stories; for Stories, we consider in-distribution dataset to be Stories and out-of-distribution dataset to be WikiText-103; for One Billion Word, we consider in-distribution dataset to be One Billion Word and out-of-distribution dataset to be enwik8; for BookCorpus, we consider  in-distribution dataset to be BookCorpus and out-of-distribution dataset to be One Billion Word; for WikiText-103, we consider in-distribution dataset to be WikiText-103 and out-of-distribution dataset to be BookCorpus. We observe that there is out-of-distribution generalization advantage in nearly all cases.


\end{document}